\documentclass{article}

\PassOptionsToPackage{numbers,sort}{natbib}

 \usepackage[main, final]{neurips_2026}

\usepackage[utf8]{inputenc} 
\usepackage[T1]{fontenc}    
\usepackage{hyperref}       
\usepackage{url}            
\usepackage{booktabs}       
\usepackage{amsfonts}       
\usepackage{nicefrac}       
\usepackage{microtype}      
\usepackage{xcolor}         

\usepackage{amsthm}
\newtheorem{theorem}{Theorem}
\usepackage{algorithm}
\usepackage{algorithmic}
\usepackage{listings}
\lstdefinelanguage{json}{
  morestring=[b]",
  morestring=[s]{'}{'},
  morecomment=[l]{//},
  morecomment=[s]{/*}{*/},
  morekeywords={true,false,null},
  sensitive=false,
  alsoletter={0123456789.,:;!_},
}
\usepackage{graphicx}
\usepackage{amsmath}
\usepackage{amssymb}
\usepackage{mathtools}
\usepackage{multirow}
\usepackage{dsfont}
\usepackage[most]{tcolorbox}
\tcbset{
  colback=gray!5, colframe=gray!80!black,
  boxrule=0.5pt, arc=2mm, outer arc=2mm,
  left=4pt, right=4pt, top=4pt, bottom=4pt,
  fonttitle=\bfseries, coltitle=white,
  before skip=6pt,
  after skip=6pt,
}
\tcbset{
  thetabox/.style={
    colback=white,       
    colframe=black!90,      
    coltitle=white,      
    boxrule=0.45pt,
    left=5pt,right=5pt,top=5pt,bottom=5pt,
    boxsep=1pt,
    before skip=1pt,
    after skip=1pt
  }
}

\title{When Should LLMs Be Less Specific? Selective Abstraction for Reliable Long-Form Text Generation}

\author{%
  Shani Goren\thanks{Equal contribution}\\
  Technion
  \And
  Ido Galil$^*$\\
  NVIDIA
  \And
  Ran El-Yaniv\\
  NVIDIA
}

\begin{document}

\maketitle

\begin{abstract}
LLMs are widely used, yet they remain prone to factual errors that erode trust and limit adoption in high-risk settings. One approach to mitigate this risk is to equip models with uncertainty estimation mechanisms that abstain when confidence is low. However, this binary ``all-or-nothing'' approach is excessively restrictive in long-form settings, often discarding valuable information. We introduce \textbf{\textit{Selective Abstraction (SA)}}, a framework that enables LLMs to trade specificity for reliability by selectively reducing the detail of uncertain content. We formalize SA through selective risk and coverage, and propose \textbf{\textit{Atom-wise Selective Abstraction}}, a claim-level instantiation that decomposes responses into atomic claims and replaces uncertain atoms with higher-confidence, less specific abstractions. We develop a novel end-to-end pipeline for evaluating long-form generation that instantiates risk as factual correctness and measures coverage as retained information. Our comprehensive empirical evaluation across five LLMs and two long-form factual benchmarks shows that atom-wise SA consistently improves the risk-coverage trade-off, outperforming seven baselines, including claim removal, self-revision, sampling-based aggregation, and chain-of-verification. It improves AURC by up to 35.61\% over claim removal, demonstrating that reducing specificity boosts reliability while preserving substantial information. We further provide a risk-guided procedure for selecting a confidence threshold that satisfies a target risk. 
Code is available at \url{https://github.com/shanigoren/selective-abstraction}.
\end{abstract}

\section{Introduction}
Large language models (LLMs) have demonstrated incredible performance across a broad spectrum of text generation tasks \cite{anthropic2024claude, comanici2025gemini25pushingfrontier, deepseekai2025deepseekv3technicalreport}. Despite these advances, LLMs remain prone to factual errors, generating content that appears fluent and coherent, but is factually incorrect or unsupported by underlying knowledge \cite{DBLP:journals/csur/JiLFYSXIBMF23, 
DBLP:journals/corr/abs-2401-01313, hallucinationsurvey,  
kalai2025languagemodelshallucinate}. 
As LLMs are increasingly deployed in high-stakes domains such as healthcare and law, factual mistakes can erode user trust and limit adoption. Ensuring trustworthy generation is therefore a central challenge for their reliable real-world use.\\
A widely studied approach to this problem is uncertainty estimation, which aims to quantify a model’s confidence in its outputs. In classification, uncertainty estimates support selective prediction \cite{El-Yaniv+Weiner}, where the model predicts only when confidence is sufficiently high and abstains otherwise. In text generation, a similar idea appears as \textit{selective generation} \cite{selectivegen}, which allows the model to abstain from answering when confidence is low. This approach is effective for short outputs such as question answering, where it is reasonable to treat the response as a single decision.\\
However, selective generation does not scale well to long-form generation. Long responses often contain many factual claims with varying levels of uncertainty, sometimes within a single sentence. For example, the sentence ``Amelia Earhart was an American aviator born on July 24, 1897'' bundles multiple claims about profession, nationality, and date of birth. An LLM may be uncertain about the exact date while remaining confident about the month, the year, or any of the other details. In a long-form setting (e.g., generating an entire biography, which may consist of dozens of sentences), abstaining from the entire generation is excessively restrictive, as it results in no content being generated at all. Even removing an entire sentence can discard valuable content. This highlights the need for a more fine-grained approach.\\
In long-form generation, a common failure mode is a single overconfident detail within an otherwise reasonable answer. In February 2023, Alphabet lost roughly \$100 billion in market value in the wake of Google’s Bard debut, as a high-profile demo response drew scrutiny for a factual mistake.~\footnote{\href{https://www.reuters.com/technology/google-ai-chatbot-bard-offers-inaccurate-information-company-ad-2023-02-08/}
{Reuters, ``Alphabet shares dive after Google AI chatbot Bard flubs answer in ad'', Feb. 8, 2023.}} Asked about James Webb Space Telescope discoveries, Bard confidently claimed JWST took the “very first pictures” of an exoplanet, an assertion off by nearly two decades. Had Bard recognized uncertainty about whether JWST was ``first'' and instead generated ``JWST has taken pictures of exoplanets'' or ``JWST has contributed important exoplanet discoveries'', the response would have remained informative while avoiding the hallucination. This incident illustrates a critical insight: not all information carries equal risk. Being overly specific in the wrong places can have dramatic consequences, while targeted abstraction, here removing just two words, maintains usefulness without sacrificing reliability.\\
We introduce \emph{\textbf{Selective Abstraction}}, a framework that enables LLMs to reduce the amount of information in long-form generations when faced with uncertainty. Rather than relying on binary all-or-nothing decisions, Selective Abstraction retains informative content by reducing the specificity of uncertain parts, yielding outputs that are less specific but more reliable. See Figure~\ref{fig:example_rc_curve} (left) for a comparison with abstention and redaction.\\

\begin{figure}[t]
    \makebox[\linewidth][c]{%
    \hspace{-60pt}%
    \begin{minipage}[t]{0.4\linewidth}
        \vspace{-2pt}
        \includegraphics[height=3.6cm]{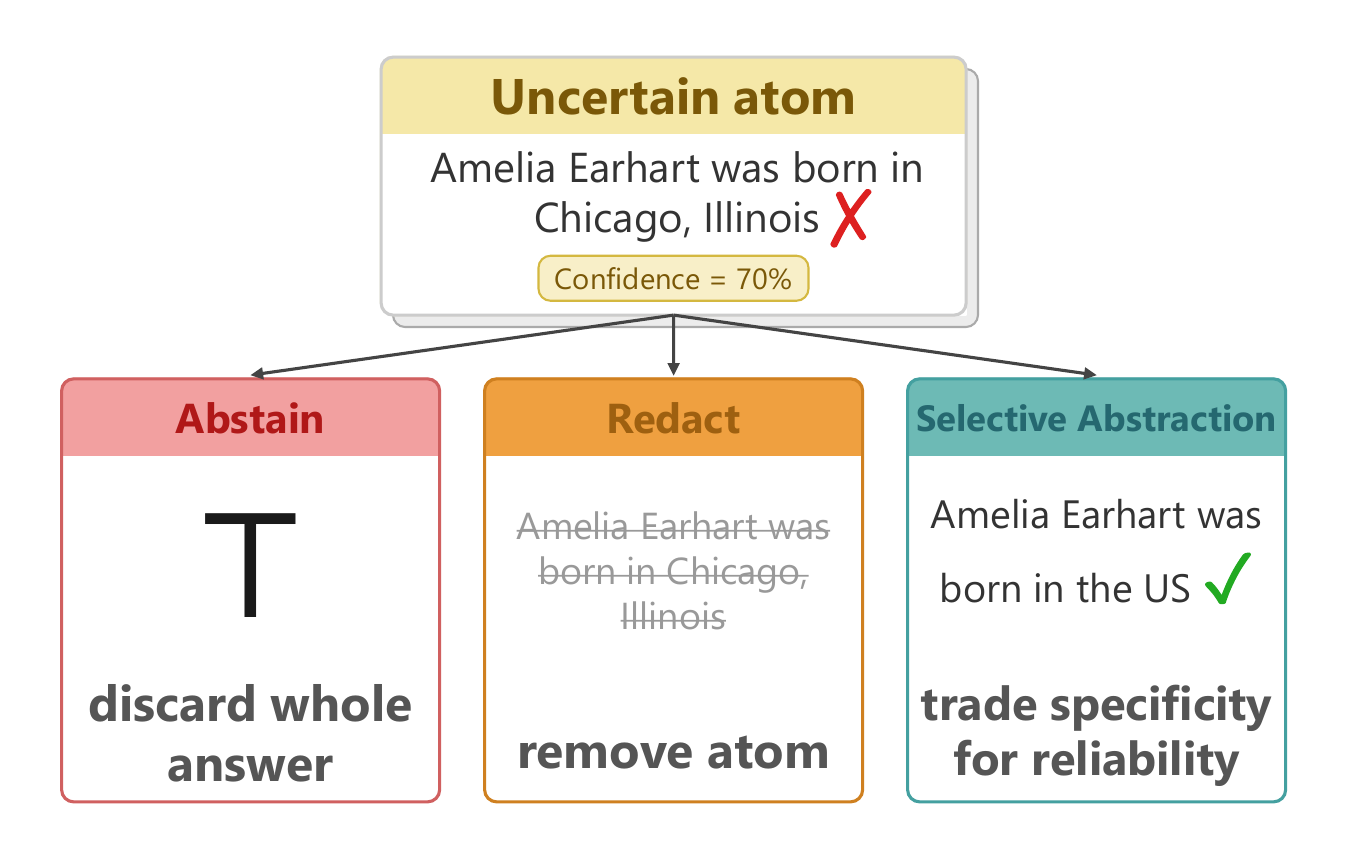}
    \end{minipage}%
    \hspace{3pt}%
    \makebox[0pt][c]{%
        \smash{%
            \color{gray!35}\rule[-3.45cm]{0.35pt}{3.45cm}%
        }%
    }%
    \hspace{2pt}%
    \begin{minipage}[t]{0.2\linewidth}
        \vspace{0pt}
        \centering
        \includegraphics[height=3.4cm]{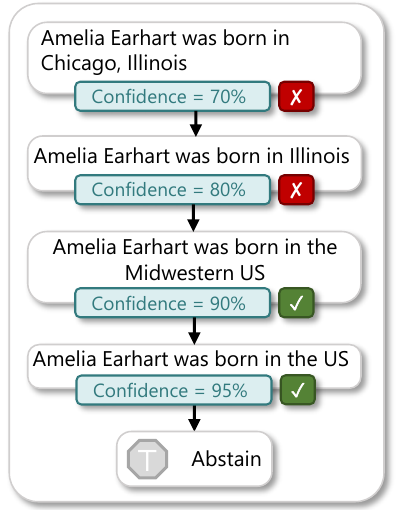}
    \end{minipage}%
    \hspace{3pt}%
    \makebox[0pt][c]{%
        \smash{%
            \color{gray!35}\rule[-3.45cm]{0.35pt}{3.45cm}%
        }%
    }%
    \hspace{2pt}%
    \begin{minipage}[t]{0.3\linewidth}
        \vspace{0pt}
        \centering
        \includegraphics[height=3.4cm]{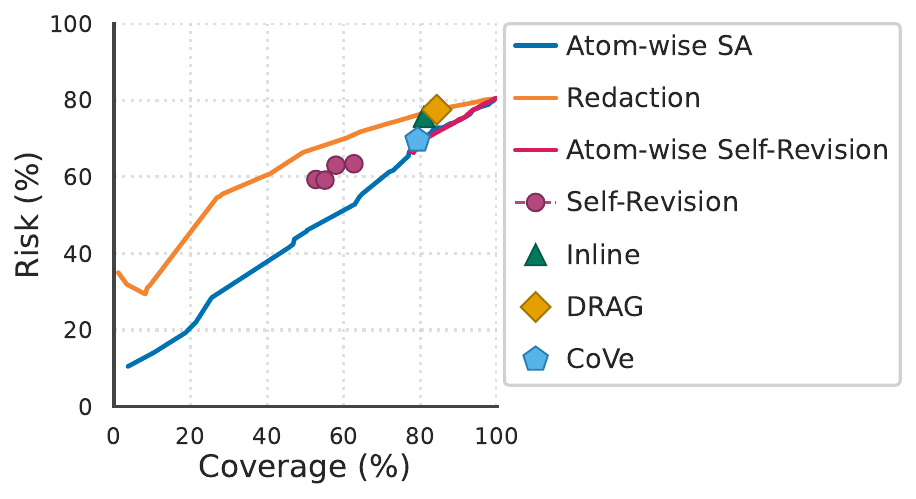}
    \end{minipage}%
    }%
    \caption{\textbf{Left:} When facing uncertainty, abstention discards the full answer, redaction removes the atom, and selective abstraction trades specificity for reliability. \textbf{Middle:} Example abstraction sequence of atom-wise SA. \textbf{Right:} Risk-coverage curve on FactScore for gpt-oss-120b. Atom-wise SA achieves lower risk than competitive baselines across the coverage range.}
    \label{fig:example_rc_curve}
    \vspace{-6pt}
\end{figure}

\textbf{The key contributions of this paper are as follows:}\\
1. We introduce \textbf{Selective Abstraction}, a general framework for trading specificity for reliability in long-form text generation, formalized through selective risk and coverage.\\
2. We propose \textbf{atom-wise selective abstraction}, a claim-level instantiation that selectively replaces uncertain claims with higher-confidence abstractions rather than removing them entirely.\\
3. We develop an \textbf{end-to-end evaluation pipeline} for selective abstraction in open-ended factual generation, instantiating risk via factual correctness and coverage via an information-theoretic measure of retained content.\\
4. We empirically demonstrate, across five LLMs and two long-form factual benchmarks, that \textbf{atom-wise selective abstraction consistently improves the risk-coverage trade-off}, outperforming seven baselines.\\
5. We propose a \textbf{risk-guided threshold selection algorithm} that selects a \emph{confidence threshold} to meet a user-defined target risk with high probability, and evaluate it empirically.

\section{Related Work}
\textbf{Selective Prediction.}
We view selective prediction as a starting point for formalizing Selective Abstraction. In classic selective prediction \cite{El-Yaniv+Weiner,geifman2019biasreduceduncertaintyestimationdeep}, a classifier may abstain when uncertain. Formally, a selective model is a pair $(f,g)$, where $f$ is a classifier and $g$ is a selection function deciding whether to predict or abstain. A common choice defines $g$ via a confidence score $\kappa$ and threshold $\theta$:
$g_\theta(x|\kappa,f)=\mathds{1}[\kappa(x,\hat{y}|f)>\theta]$, where $\hat{y}=f(x)$.
Thus, the model predicts only when its confidence exceeds $\theta$.\
Selective prediction induces a trade-off between \emph{risk} and \emph{coverage}. Risk is defined with respect to a given loss, e.g., the 0/1 loss, and is measured only over the samples that are predicted. Coverage is the fraction of samples that are not rejected. By rejecting uncertain samples, the selective model aims to reduce risk at the cost of lower coverage. 
\emph{Hierarchical selective classification} \cite{hsc} extends selective classification to a hierarchical setting. Under uncertainty, the model can retreat to a less specific ancestor label rather than abstaining (e.g., predicting ``Dog'' instead of ``Golden Retriever''), thereby using the hierarchy to guide a reduction in the information content of predictions. This is conceptually aligned with our objective of trading specificity for reliability, rather than relying on a binary abstention decision.

\textbf{Selective Generation.} 
Most existing work on uncertainty estimation for text generation focuses on tasks where the generated text is at most a few tokens long, such as machine translation, text summarization, and question answering (QA)~\cite{lmpolygraph, semantickuhn, DBLP:conf/nips/0001YPWWY0T24}. Among these, QA is the most relevant to our setting. However, even in open-ended QA benchmarks, the generated answers are typically short, e.g. TriviaQA~\cite{triviaqa} and CoQA~\cite{coqa}. This allows for a binary decision: either accept or reject the model’s output. This approach, known as \emph{selective generation}~\cite{selectivegen}, serves as the NLP analogue of selective classification~\cite{lmpolygraph, DBLP:conf/acl/FengS00BT24, DBLP:journals/tmlr/LinT024, DBLP:journals/tacl/VashurinFVRVTPXSGPBNPS25, DBLP:journals/corr/abs-2502-14427}. While these tasks are considerably easier to evaluate, the uncertainty estimation methods developed for them do not generalize well to long-form text generation \cite{DBLP:conf/emnlp/Zhang0BC24}, which makes for a much more realistic use case for LLMs.

\begin{figure*}[t]
    \centering
    \includegraphics[width=\linewidth]{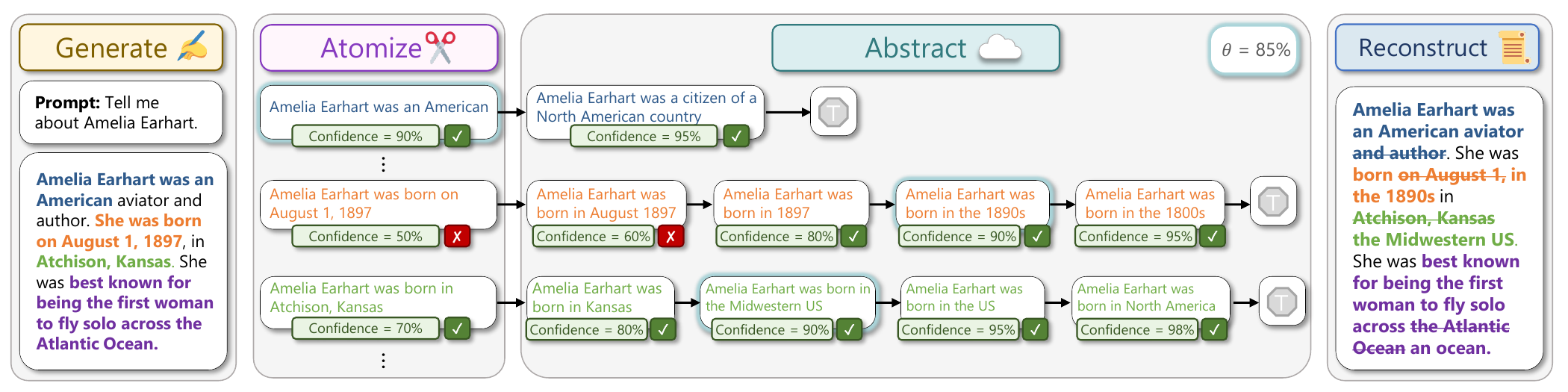}
    \caption{
    An overview of the Selective Abstraction framework. The generated text is decomposed into atoms, and low-confidence atoms are replaced with confident abstractions. The confidence threshold in this example is 85$\%$.
    }
    \vspace{-12pt}
\label{fig:abstraction_pipeline}
\end{figure*}

\textbf{Evaluation of Long-form Generation.}
In our work, quantifying \emph{risk} (Section~\ref{subsec:risk}) requires evaluating the reliability of the generated response. Throughout this paper, we use factual correctness as our notion of reliability. Most existing work in this research thread is focused on mitigating hallucinations in long-form text via fact-checking pipelines. These methods typically decompose model outputs into smaller units and verify them using external tools, such as retrieval systems or auxiliary LLMs.
Earlier frameworks such as FactScore laid important groundwork but are constrained by their design: FactScore \cite{FactScore} relies on static retrieval and entity-level assumptions, while SAFE \cite{DBLP:conf/nips/WeiYSLH0TPLHDL24} depends on the costly Google Search API and follows a rigid Google-Search loop with a predetermined number of queries per claim. 
In contrast, our fact-checking agent (Section ~\ref{subsec:risk}) supports unconstrained Wikipedia search, and relies only on free publicly available resources, enabling large-scale evaluation.\\
An agentic approach to factuality evaluation has recently been adopted by OpenAI\footnote{https://cdn.openai.com/gpt-5-system-card.pdf}, which verifies claims using commercial models accessed via a paid API.
To reduce cost, they verify multiple claims within a single LLM call, trading off per-claim focused verification, since evidence and reasoning must be shared across claims and can degrade performance \cite{ji2025batchprompting}. In contrast, our agent verifies each abstraction independently, enabling targeted evidence gathering and focused per-claim reasoning.

\textbf{Uncertainty-Aware Fact-Checking.}
Other works explicitly incorporate uncertainty estimation into fact-checking pipelines. \citet{DBLP:journals/corr/abs-2307-03987} detect LLM hallucinations by extracting key parts of the output using an external model, and verify the uncertain parts against an external knowledge source. \citet{DBLP:conf/acl/PanWLLWKN23} proposed to fact-check complex statements by decomposing them into simpler subtasks and generating reasoning programs to verify these statements. ~\citet{DBLP:conf/acl/FadeevaRSPLMTKP24} identify uncertain tokens in long-form outputs, though they stop short of editing the generation. 

\textbf{Granularity-driven Generation and Evaluation.} 
Recent work has begun to explore the relationship between granularity, uncertainty, and factuality in language model outputs, though mostly in narrow contexts.
\citet{yona2024narrowingknowledgeevaluationgap} observe that factual questions can be answered correctly at multiple levels of granularity, yet standard QA evaluation considers only a single reference level, penalizing less specific yet still correct responses. They find that models tend to be overly specific and propose reducing granularity by aggregating multiple sampled answers into a higher-level response. While conceptually aligned with our view that reducing information can mitigate factual errors, their approach is limited to QA and relies on sampling, making it impractical for long, open-ended generation.\\
They further introduce a relative informativeness metric defined only within alternative granularities of the same answer. In contrast, our information measure is computed independently of other claims, making it interpretable and applicable to open-ended generation, and consequently allowing comparisons across statements of varying granularity that are not related by abstraction.\\
Other works treat specificity as a one-sided metric, emphasizing maximal detail while overlooking its trade-off with correctness under uncertainty. 
In conversational systems, specificity is often regarded as inherently desirable for making interactions more engaging, and annotators are guided to rate responses such that specific ones are preferred \cite{adiwardana2020humanlikeopendomainchatbot, thoppilan2022lamdalanguagemodelsdialog}. This narrow view also appears in factual QA settings \cite{ kahl-etal-2024-llms, sui2025knowledgegraphsmakelarge}, where both human and LLM-as-a-judge annotators are instructed to favor more specific outputs. \citet{huang2023languagemodelsspecifichow} evaluate specificity via masked-token prediction, but their formulation is restricted to narrow outputs and does not capture variable granularity in longer text.

\section{Selective Abstraction}\label{sec:selective_abstraction}
\emph{\textbf{Selective abstraction (SA)}} reduces the specificity of uncertain information in generated text, with the goal of improving its reliability. 
In this section, we formalize the selective abstraction framework at a conceptual level.
Let $\mathcal{M}$ be a language model. Given a prompt $p$, denote the language model's generated text response: $\mathcal{M}(p) = T$.\footnote{Formally, $\mathcal{M}(p)$ induces a distribution over responses and $T \sim \mathcal{M}(p)$. Throughout this paper, we write quantities that depend on $T$ (e.g., risk, coverage and information) as functions of a realized response, but they should be interpreted as expectations over the generation randomness.}\\
We define selective abstraction as:
$ SA(p,\mathcal{M}) = T'$, 
a process that uses $\mathcal{M}$ to output $T'$, a response to $p$ such that, ideally $\text{Accuracy}(T') > \text{Accuracy}(T)$, at the possible cost of information loss, $I(T') \le I(T)$ where $I(\cdot)$ denotes an information measure over texts (see Section~\ref{subsec:coverage} for definition).
We keep this definition broad to encompass any mechanism that uses uncertainty to reduce the amount of information expressed in the model's response.\\
Selective abstraction can be performed using an explicit confidence scoring function, denoted $\kappa(t|\mathcal{M})$, that numerically quantifies the model's confidence assigned to a text $t$ (where 
$t$ is any text string, including but not limited to the generated response). For a detailed discussion of the role of the confidence function, see Appendix~\ref{app:kappa}. A threshold $\theta$ may also be specified to control the degree of information reduction: when $(\kappa,\theta)$ are specified, the resulting response $T'$ is required to satisfy $\kappa(T') \ge \theta$. We denote this by
$SA(p,\mathcal{M},\kappa,\theta)=T'$.\\
We evaluate the abstractive behavior of SA using \emph{risk} and \emph{coverage}.
Risk can be quantified via any user-chosen loss function that reflects reliability, for instance, the hallucination rate.
Coverage is defined as the fraction of retained information,
$
\phi(T',T) = \frac{I(T')}{I(T)}.
$
Varying the threshold $\theta$ induces a risk-coverage trade-off. Higher thresholds will reduce coverage while decreasing risk. In Section~\ref{sec:evaluation}, we formalize the concrete risk and coverage definitions used in our experiments.
Figure~\ref{fig:qualitative_example} shows a concrete example of this trade-off: the higher threshold reduces detail but improves reliability, whereas the lower threshold retains more detail at higher risk.

\begin{figure*}[t]
    \centering
    \includegraphics[width=\linewidth]{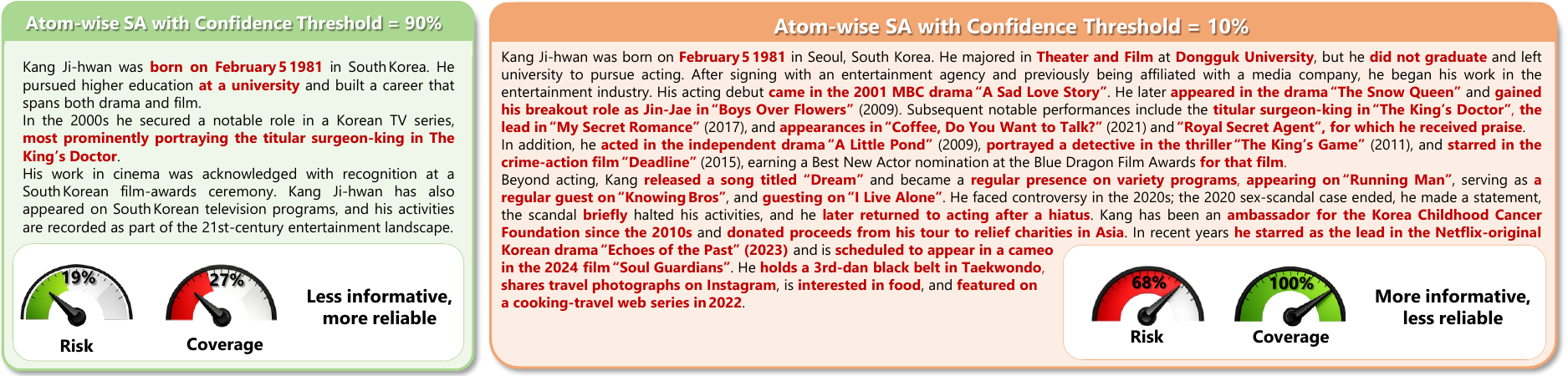}
    \caption{
    Qualitative example of atom-wise Selective Abstraction at two confidence thresholds: a higher threshold reduces detail but improves reliability, while a lower threshold retains detail at higher risk. \textcolor{red}{Red text} signifies incorrect claims. Generating model: gpt-oss-120b, dataset: FactScore. See Appendix~\ref{app:sa_examples} for additional examples.
    }
\label{fig:qualitative_example}
\end{figure*}

\subsection{Atom-wise Selective Abstraction}\label{sec:atom_wise_sa}
Long-form responses contain many distinct pieces of information, so selective abstraction must operate at a finer granularity. 
We therefore propose \textit{Atom-wise Selective Abstraction}, a claim-level instantiation of SA, as a procedure that enables control of information reduction.
Concretely, we selectively replace low-confidence claims with less specific, higher-confidence abstractions. Figure~\ref{fig:abstraction_pipeline} illustrates the procedure.
All prompts used throughout this section are provided in Appendix \ref{app:sa_prompts}. 

\textbf{Stage 1: Generation.}
Given a prompt $p$, we generate an initial response $T=\mathcal{M}(p)$. 

\textbf{Stage 2: Atomization.}
We follow prior work \cite{DBLP:conf/naacl/NenkovaP04, shapira-etal-2019-crowdsourcing, zhang-bansal-2021-finding, liu2023revisitinggoldstandardgrounding, FactScore} and decompose the text into \emph{atomic claims}\footnote{Throughout this paper, we use the terms \emph{atomic claims} and \emph{atoms} interchangeably.}, which are short, standalone statements, each expressing a single piece of information.
Let $\mathcal{A}_{T} = (a_i)_{i=1}^{n}$ denote the set of atoms extracted from $T$. The decomposition is performed by prompting the generating LLM with few-shot examples.

\textbf{Stage 3: Abstraction.}
This is the core stage of our framework. It operates over the atoms produced in Stage 2 and consists of three steps: (1) atom confidence elicitation, (2) abstraction generation, and (3) selection.

\textbf{(1) Atom Confidence Elicitation.} 
We obtain the model's confidence score for each atom. Let  $\kappa_{\text{atom}}(a_i\mid p,T,\mathcal{M})$ denote a confidence scoring function that quantifies the confidence of the model $\mathcal{M}$ in $a_i$, conditioned on the prompt $p$ and the original response $T$.
The primary confidence function we use throughout the paper is \emph{verbalized confidence}, where the LLM is explicitly prompted to report its confidence. However, our method can be used with any confidence function. For discussion on alternative confidence functions, see Appendix~\ref{app:kappa}. For each response, the model is provided with the full list of atoms in a single prompt and asked to assign a confidence score to each atom. 

\textbf{(2) Abstraction Generation and Confidence Elicitation.} 
Let $\theta$ denote a confidence threshold.
For each atom with $\kappa_{\text{atom}}(a_i\mid p,T,\mathcal{M})<\theta$, we prompt the generating model to produce an abstraction sequence:
$
Ab(a_i)=\bigl(a_i^{(0)},a_i^{(1)},\ldots,a_i^{(j)},a_i^{(j+1)}=\top\bigr),
$
where $a_i^{(0)}=a_i$, each $a_i^{(j+1)}$ is less specific than $a_i^{(j)}$, and $\top$ corresponds to full abstention (a zero-information abstraction). 
We then obtain an abstraction-level confidence score $\kappa_{\text{abs}}(a_i^{(j)}\mid p,T,\mathcal{M})$ for each abstraction in the sequence by prompting the model to score the list of abstractions, conditioned on the atom confidence $\kappa_{\text{atom}}(a_i\mid p,T,\mathcal{M})$ (We discuss the motivation for this choice in Appendix~\ref{app:design_choices}, and provide an additional manual validation of the generated abstraction sequences in Appendix~\ref{app:abstraction_validity}.).
Figure~\ref{fig:example_rc_curve} (middle) shows an example abstraction sequence.

\textbf{(3) Selection.}  
After obtaining confidence scores for all abstractions, we perform a selection step that determines the appropriate level of abstraction for each atom. The goal of this step is to reduce information only as needed, ensuring that the confidence of each selected abstraction exceeds the threshold $\theta$.
Formally, for each atom $a_i$, select the most specific abstraction $a_i^{(j^*)}$ that satisfies the threshold:
$a_i^{(j^*)} \quad \text{s.t.} \quad 
j^* = \min \left\{ j \ge 0 :
\kappa_{\text{abs}} \left(a_i^{(j)} \mid \mathcal{M}, p, T \right) > \theta
\right\}.$
If no such $j$ exists, abstain: $a_i^{(j^*)} = \top$.
This step yields a single selected abstraction (or abstention) per atom.

\textbf{Stage 4: Reconstruction.} 
In the final step, we prompt the LLM to reconstruct the selected abstractions into a coherent text. $T'$, yielding a revised version of $T$ that preserves high-confidence content while abstracting uncertain details. 

\vspace{-0.5em}
\section{Evaluating Selective Abstraction}\label{sec:evaluation}
Evaluating selective abstraction requires quantifying both risk and information, which is inherently difficult in open-ended generation. In such settings, a ground truth may not exist, and even when it does, there is often no clear way to quantify how much information remains. Although selective abstraction can, in principle, be applied to any kind of text generation, defining meaningful measures of risk, and especially information, is not always possible. 
We therefore focus our quantitative evaluation on factual long-form generation, where both risk and information can be measured at scale. This setting captures the central use case motivating selective abstraction: reducing factual risk while preserving as much useful content as possible. 
To illustrate the broader applicability of our framework, Appendix~\ref{app:non_factual_examples} provides qualitative examples beyond factual generation.
In this section, we specify the concrete instantiations of \emph{risk} and \emph{information} used for the factual evaluations in this paper.
The risk and information measures described below are used only for offline evaluation, and are not computed as part of the Selective Abstraction inference procedure.

\subsection{Risk}\label{subsec:risk}
Let $A(T) = (a_i)_{i=1}^{n}$ be the atomic claims extracted from a response $T$, and define $\mathrm{correct}(a_i) \in \{0,1\}$ by $\mathrm{correct}(a_i) = 1$ if and only if $a_i \text{ is supported by evidence}$.
We then define risk on $T$ as:
$
R(T) = 1 - \frac{1}{n}\sum_{i=1}^{n} \mathrm{correct}(a_i).
$
To annotate  $\mathrm{correct}(a_j)$, we adopt an agentic fact-checking approach that verifies abstractions against Wikipedia.
The agent searches for evidence supporting each abstraction and labels it as \emph{supported} when such evidence is found, or as \emph{unsupported} when the abstraction is explicitly refuted or when no supporting evidence can be identified (prompts are provided in Appendix~\ref{app:correctness_agent}).
Our approach relies only on publicly available resources, using the free Wikipedia API\footnote{\href{https://www.mediawiki.org/wiki/API:Action_API}{https://www.mediawiki.org/wiki/API:Action\_API}} together with the open-source model gpt-oss-120b \cite{openai2025gptoss120bgptoss20bmodel} as the agent’s backbone. This enables dynamic, reasoning-based annotation while supporting large-scale evaluation. 
We validate the correctness agent on 500 manually annotated abstractions, achieving high precision ($0.922$) and recall ($0.905$); see Appendix~\ref{app:correctness_validation} for further details.

\subsection{Information and Coverage}\label{subsec:coverage}
Let $A(T) = (a_i)_{i=1}^{n}$ and 
$A(T') = (a'_j)_{j=1}^{m}$ denote the sets of atomic claims extracted from an initial response $T$ and an abstracted response $T'$, respectively. We define the information content of a text as the sum of the information of its distinct atomic claims after deduplication. We then define the \textit{coverage} of $T'$ as the information retained in $T'$ relative to $T$:
$
\phi(T',T) = \frac{I(T')}{I(T)} = \frac{\sum_{j=1}^{m}I(a'_j)}{\sum_{i=1}^{n}I(a_i)}.
$\\
Importantly, our framework is agnostic to the definition of information: an alternative information measure can be plugged in without altering SA itself. In fact, SA (and the threshold selection algorithm in Section~\ref{sec:threshold_alg}) can be applied without any information measure.

We quantify the information encoded in an atomic claim by estimating how strongly it narrows down the set of compatible entities, that is, how much new information it contributes about an entity. 
For example, consider the claim \emph{``Alan Turing was born in London''}. We view it as specifying a subject-independent condition, \emph{born in London}, and thus as defining a subset of entities that satisfy this condition. Relative to a broad reference set, \emph{People}, the claim restricts the compatible entities from all people to the smaller subset of people born in London.
Practically, this corresponds to asking how many people were born in London out of all people, and using that ratio as a proxy for how informative the claim is.\\
Formally, each claim $c$ has the form $[\mathrm{entity}] [\mathrm{predicate}]$. We measure information by interpreting the $\mathrm{predicate}$ span as a subject-independent condition over entities of the appropriate type.
This induces: (i) a domain of discourse $\mathcal{E}$, the set of all entities of the same type as $\mathrm{entity}$ (e.g., if $\mathrm{entity}$ is a person, then $\mathcal{E} = \{\text{all people}\}$), and (ii) a predicate $P_c : E \to \{0, 1\}$ derived by removing the specific entity from $c$. 
We then define:
$
\mathcal{E}_c \;=\; \{\, e \in \mathcal{E} \;:\; P_c(e)=1\,\}, \mathcal{E}_c \subseteq \mathcal{E}.
$
In our example, for $c=$ ``Alan Turing was born in London'', $\mathcal{E}$ is the set of all people, and $\mathcal{E}_c$ is the subset of people born in London.\\
We adopt an information-theoretic perspective and measure the uncertainty associated with each set by computing its entropy. Assuming a uniform prior\footnote{This is a simplifying assumption with known limitations, since real-world entity frequencies are not uniform, but it enables estimating entropy from set size alone without introducing assumptions about how likely different entities are.} within each set, the entropy of a set $\mathcal{E}$ is: 
$H(\mathcal{E})=-\sum_{e\in \mathcal{E}} \frac{1}{|\mathcal{E}|} \log {\frac{1}{|\mathcal{E}|}} = \log ({{|\mathcal{E}|}}).$
We then define the information encoded in $c$ as the relative reduction in entropy:
$
I(c) = 1-\frac{H(\mathcal{E}_c)}{H(\mathcal{E})} = 1-\frac{\log ({{|\mathcal{E}_c|}})}{\log ({{|\mathcal{E}|}})}.
$
$I(c)$ increases as $c$ narrows the set of compatible entities. By construction, trivial claims ($\mathcal{E}_c=\mathcal{E}$, for instance: ``X is a person'') carry zero information. \\
Our notion of information aligns with prior work that views information as uncertainty reduction~\cite{Bar-Hillel1953-BARSI,   carnap_outline_1952, DBLP:journals/corr/Resnik95a}.
This perspective also offers a useful intuition: our measure resembles information gain in decision-tree splits.
Each atomic claim narrows the set of compatible entities, and $I(c)$ quantifies the resulting uncertainty reduction. We emphasize that this analogy is given for intuition only: unlike decision trees on a fixed dataset, our splits are defined over an open-world entity set.\\
To estimate $|\mathcal{E}_c|$ and $|\mathcal{E}|$ automatically, we use  Wikidata \cite{wikidata}, one of the largest publicly available knowledge bases, which stores structured information about real-world entities and can be queried using SPARQL, a standard query language for structured knowledge bases.
We use the SPINACH agent \cite{spinach}, which translates natural-language questions into SPARQL queries and executes them against the Wikidata Query Service to estimate the sizes of the corresponding entity sets. Since obtaining verified real-world counts is infeasible in an open-world setting, we use Wikidata as a best-available proxy for these set sizes.\\
Continuing the example, we query Wikidata for the total number of people ($12{,}352{,}844$) and for the number of people born in London ($20{,}313$). Using these values,
$
I(c)=1-\frac{\log(20{,}313)}{\log(12{,}352{,}844)} \approx 0.39.
$
For comparison, \emph{``born in England''} corresponds to $135{,}340$ entities, yielding $I\approx 0.28$, consistent with \emph{``born in London''} being more specific than \emph{``born in England''}.\\



\subsection{RC Curves and AURC}
Selective abstraction induces a trade-off between reliability and specificity, measured by risk and coverage. We summarize this trade-off using \emph{Risk-Coverage (RC) curves} \cite{El-Yaniv+Weiner}, which plot the risk of retained content as a function of coverage. An example RC curve is shown in Figure~\ref{fig:example_rc_curve} (right).\\
Given a confidence scoring function $\kappa(\cdot)$ and a threshold $\theta$, $SA$ produces an abstracted response $T'(\theta)$. For each $\theta$, we compute its coverage $\phi(T'(\theta), T)$ and risk $R(T'(\theta))$, as defined in the preceding subsections. Varying $\theta$ over the range of $\kappa$ traces an RC curve.\\
We summarize each RC curve using the \emph{area under the RC curve (AURC)}, defined as
$
\mathrm{AURC} \;=\; \int_{0}^{1} R(\phi)\, d\phi,
$
where $R(\phi)$ denotes the risk achieved at coverage $\phi$. In practice, since evaluation is performed on a finite set of examples, we sweep over all distinct attainable thresholds $\{\theta_k\}_{k=1}^{K}$, induced by the confidence scores in the evaluation set, which yields a finite sequence of operating points with coverages $\{\phi_k\}$ (sorted by increasing order). 
Lower AURC indicates a better risk-coverage trade-off because it corresponds to lower risk on average across the full coverage range.

\begin{table*}[t]
\vskip -0.15in
  \caption{AURC comparison against curve-producing baselines. AURC is computed over the coverage interval shared by Atom-wise SA and each baseline. Lower AURC is better, and positive improvement indicates lower AURC for Atom-wise SA. Table~\ref{tab:aurc-bootstrap} (App.~\ref{app:aurc_results}) reports 95\% bootstrap CIs.}
  \label{tab:aurc}
  \begin{center}
      \begin{sc}
      \resizebox{\textwidth}{!}{
\begin{tabular}{llcccccc}
\toprule
\multirow{3}{*}{Benchmark} & \multirow{3}{*}{Model} & \multicolumn{3}{c}{Aw-SA vs. Redaction} & \multicolumn{3}{c}{Aw-SA vs. Aw-SelfRev} \\
\cmidrule(lr){3-5}\cmidrule(lr){6-8}
& & Redaction & Aw-SA & Improvement & Aw-SelfRev & Aw-SA & Improvement \\
& & AURC $\downarrow$ ($\times 100$) & AURC $\downarrow$ ($\times 100$) & $\uparrow$ (\%) & AURC $\downarrow$ ($\times 100$) & AURC $\downarrow$ ($\times 100$) & $\uparrow$ (\%) \\
\midrule
\multirow{6}{*}{FactScore} & gpt-5.2 & $11.64$ & $\mathbf{7.49}$ & $35.61$ & $4.20$ & $\mathbf{3.47}$ & $17.28$ \\
 & gpt-oss-120b & $60.43$ & $\mathbf{44.56}$ & $26.26$ & $\mathbf{16.68}$ & $16.74$ & $-0.36$ \\
 & gpt-oss-20b & $70.96$ & $\mathbf{59.03}$ & $16.81$ & $42.60$ & $\mathbf{37.78}$ & $11.31$ \\
 & qwen3-30b-t & $46.99$ & $\mathbf{40.31}$ & $14.22$ & $9.84$ & $\mathbf{9.53}$ & $3.11$ \\
 & qwen3-4b-t & $45.90$ & $\mathbf{41.09}$ & $10.47$ & $10.03$ & $\mathbf{9.60}$ & $4.34$ \\
 & \textbf{Average} & $47.19$ & $\mathbf{38.50}$ & $20.68$ & $16.67$ & $\mathbf{15.42}$ & $7.13$ \\
\midrule
\multirow{6}{*}{LongFact} & gpt-5.2 & $13.99$ & $\mathbf{12.50}$ & $10.65$ & $3.82$ & $\mathbf{3.30}$ & $13.62$ \\
 & gpt-oss-120b & $36.42$ & $\mathbf{28.99}$ & $20.40$ & $6.10$ & $\mathbf{5.63}$ & $7.59$ \\
 & gpt-oss-20b & $32.38$ & $\mathbf{27.86}$ & $13.95$ & $9.15$ & $\mathbf{8.03}$ & $12.19$ \\
 & qwen3-30b-t & $15.79$ & $\mathbf{14.02}$ & $11.25$ & $3.79$ & $\mathbf{3.68}$ & $2.97$ \\
 & qwen3-4b-t & $30.27$ & $\mathbf{27.26}$ & $9.93$ & $6.97$ & $\mathbf{6.62}$ & $5.06$ \\
 & \textbf{Average} & $25.77$ & $\mathbf{22.13}$ & $13.24$ & $5.97$ & $\mathbf{5.45}$ & $8.28$ \\
\bottomrule
\end{tabular}}
      \end{sc}
  \end{center}
  \vskip -0.15in
\end{table*}

\section{Experimental Setup}\label{sec:exp_setup}
\textbf{Models.}
We evaluate both proprietary and open-weight reasoning models: GPT-5.2 \cite{openai2025gpt52}, gpt-oss-120b, gpt-oss-20b \cite{openai2025gptoss120bgptoss20bmodel}, Qwen3-30B Thinking and Qwen3-4B Thinking \cite{qwen3technicalreport}.

\textbf{Datasets.}
We conduct our evaluations on responses to prompts taken from two widely used long-form factuality benchmarks: FactScore~\cite{FactScore} and LongFact-Objects~\cite{DBLP:conf/nips/WeiYSLH0TPLHDL24} (abbreviated LongFact). FactScore consists of 183 prompts that ask for biographies of notable individuals. LongFact-Objects contains LLM-generated, open-ended questions across 38 subject areas (e.g., astronomy, biology, and computer science), which elicit detailed long-form answers.
For LongFact, we sample 76 prompts evenly across subject areas.

\textbf{Baselines.}
We compare atom-wise Selective Abstraction against the following seven baselines (prompts for all baselines are provided in Appendix~\ref{app:baseline_prompts}):\\
(1) \textbf{Selective Redaction}: removes all atoms whose confidence falls below a threshold, without producing abstractions. This baseline ablates abstracting uncertain claims.\\
(2) \textbf{Self-Revision (SelfRev)}: prompts the model to revise its full original response by removing uncertain claims or replacing them with less specific, higher-confidence claims.\\
(3) \textbf{Atom-wise Self-Revision (Aw-SelfRev)}: applies self-revision at the level of atomic claims. We decompose the response into atoms and revise atoms whose confidence falls below a threshold. This baseline ablates the ordered abstraction sequence.\\
(4) \textbf{Iterative Self-Revision (Iter-SelfRev)}: repeatedly applies the Self-Revision prompt, where the output of iteration $k-1$ is used as the input to iteration $k$. We evaluate iterations up to $k=4$.\\
(5) \textbf{Inline Abstraction}: adapts the \textsc{Agg} prompting baseline from ~\cite{yona2024narrowingknowledgeevaluationgap}, which reduces information during generation by instructing the model to respond at a granularity that matches its uncertainty. This ablates post-hoc editing, testing whether uncertainty-aware abstraction can be handled directly during generation.\\
(6) \textbf{Decoding with
Response Aggregation (DRAG)}~\cite{yona2024narrowingknowledgeevaluationgap}: samples $N=5$ responses and aggregates them into a single coherent response, preserving consistent information across the sampled responses. We adapt this sample-then-aggregate procedure from short-form QA to long-form generation.\\
(7) \textbf{Chain-of-Verification (CoVe)}~\cite{dhuliawala2023chainofverificationreduceshallucinationlarge}: applies a generate-then-verify procedure to long-form responses. 
We include CoVe as a verification-based baseline that attempts to improve factuality by post-hoc verification, rather than by reducing specificity according to confidence.

Additional implementation, reproducibility, and compute details are provided in Appendix~\ref{app:technical_details}.


\section{Results}
Across all models and datasets, atom-wise selective abstraction (Aw-SA) achieves a better risk-coverage trade-off than the competing approaches considered in this work.
This advantage is evident in summary metrics, through lower AURC (Table~\ref{tab:aurc}), and is also visible in the risk-coverage curves (Figures~\ref{fig:FactScore_rc_curves_grid} and~\ref{fig:longfact_rc_curves_grid} in Appendix~\ref{app:full_results}).
These results suggest that gradual, atom-wise reduction is an effective and controllable mechanism for reducing information in a way that improves reliability while preserving coverage. \\
Because AURC is defined over a risk-coverage curve, Table~\ref{tab:aurc} includes baselines that produce a curve: Selective Redaction and Atom-wise Self-Revision. The remaining baselines produce single operating points rather than full curves, and are therefore evaluated separately at matched coverage (Table~\ref{tab:dot_vs_sa_gap_combined}).\\
Compared to \textbf{Selective Redaction}, Aw-SA consistently lowers AURC across all models and datasets. On FactScore, Aw-SA improves AURC by 20.68\% on average, with a maximal improvement of 35.61\% with GPT-5.2. On LongFact, the average improvement is 13.24\%, with a maximal improvement of 20.40\% with gpt-oss-120b. This gap demonstrates the advantage of abstraction over removal.
Among all evaluated baselines, the most competitive is \textbf{Atom-wise Self-Revision}, our atomized variant of Self-Revision. Aw-SA improves average AURC over this baseline by 7.13\% on FactScore and 8.28\% on LongFact, demonstrating the advantage of ordered abstraction sequences over unconstrained atom rewrites.
The risk-coverage curves (Figures~\ref{fig:FactScore_rc_curves_grid} and~\ref{fig:longfact_rc_curves_grid}) further confirm these gains: atom-wise SA yields a smooth trade-off that dominates the other baselines across the full coverage range, consistently achieving lower risk at the same coverage.\\

The \textbf{Inline Abstraction}, \textbf{DRAG}, \textbf{CoVe}, \textbf{Self-Revision}, and
\textbf{Iterative Self-Revision} baselines yield single operating points rather than full
risk-coverage curves. Evaluating them across the full range of $\theta$ would require separate generations for each threshold, which is infeasible at our scale. We report thresholded variants of the Inline and Self-Revision baselines, which directly inject $\theta$ into the prompts, but find that they do not produce a meaningful trade-off (Figure~\ref{fig:threshold_baselines} in Appendix~\ref{app:threshold_baselines}). 
We therefore compare these baselines to atom-wise SA by reporting the risk difference at matched coverage (Table~\ref{tab:dot_vs_sa_gap_combined}). Across both datasets, atom-wise SA consistently achieves markedly lower risk at matched coverage. The average
gap ranges from 11.15\% to 20.50\% on FactScore, and from 4.47\% to 26.07\% on LongFact.
Observing the RC curves (Figures~\ref{fig:FactScore_rc_curves_grid} and
\ref{fig:longfact_rc_curves_grid} in Appendix~\ref{app:full_results}), these single-point
baselines generally lie closer to the Redaction curve than to the atom-wise SA curve, and
in some cases even exceed Redaction. Coverage above 100\% indicates that the transformed
response introduced additional information beyond the original response.

\begin{table}[t]
  \caption{Risk gap at matched coverage between single-point baselines and atom-wise SA (FactScore / LongFact). 
  Positive values mean lower risk for atom-wise SA. Atom-wise SA achieves lower risk at matched coverage.}
    \vskip -0.1in
  \label{tab:dot_vs_sa_gap_combined}
  \begin{center}
    \begin{sc}
    \resizebox{\textwidth}{!}{
\begin{tabular}{lcccccc}
\toprule
Model & Inline \cite{yona2024narrowingknowledgeevaluationgap} & DRAG \cite{yona2024narrowingknowledgeevaluationgap} & CoVe \cite{dhuliawala2023chainofverificationreduceshallucinationlarge} & Self-Revision & Iter. SelfRev $k{=}2$ & Iter. SelfRev $k{=}4$ \\
\midrule
gpt-5.2 & 14.14 / 1.69 & 11.70 / 30.69 & 20.66 / 39.57 & 11.01 / 0.04 & 27.98 / 24.83 & 28.40 / 27.97 \\
gpt-oss-120b & 5.59 / 10.41 & 4.69 / 19.39 & 0.80 / 15.84 & 10.72 / 3.92 & 12.97 / 11.47 & 10.67 / 9.86 \\
gpt-oss-20b & 7.38 / 11.32 & 2.36 / 14.04 & 19.25 / 7.58 & 12.19 / -1.93 & 15.77 / 6.66 & 11.10 / 5.43 \\
qwen3-30b-t & 16.42 / 4.79 & 36.85 / 40.64 & 24.91 / 25.00 & 14.36 / 4.40 & 25.70 / 28.83 & 24.48 / 26.44 \\
qwen3-4b-t & 16.58 / 19.86 & 25.25 / 25.59 & 12.18 / 14.96 & 7.44 / 15.92 & 18.50 / 16.68 & 22.72 / 22.54 \\
\textbf{Average} & 12.02 / 9.61 & 16.17 / 26.07 & 15.56 / 20.59 & 11.15 / 4.47 & 20.18 / 17.70 & 19.47 / 18.45 \\
\bottomrule
\end{tabular}}
    \end{sc}
  \end{center}
  \vskip -0.2in
\end{table}

Finally, we provide several additional analyses in the appendix. 
Table~\ref{tab:abstraction_stats} in Appendix~\ref{app:response_stats} reports response-level statistics, showing that generations are information-rich and that each atomic claim supports multiple abstraction steps, offering plentiful opportunities for gradual information reduction. 
Since Aw-SA applies abstraction decisions at the atom level, Appendix~\ref{app:dependency_robustness} analyzes its effect on cross-atom dependencies, showing that context dependencies are highly preserved.

\vskip -2in
\section{Threshold Selection for Target Risk}\label{sec:threshold_alg}
Thus far, we have studied the trade-off between risk and coverage induced by atom-wise SA. 
However, practical deployment typically requires fixing a specific confidence threshold $\theta$, which calls for a principled method for choosing it.
In this section, we propose an algorithm for choosing $\theta$ to satisfy a user-specified target risk.
Our approach is inspired by split conformal prediction \cite{conformal_orig}, using a calibration set to select a confidence threshold that, with high probability, guarantees that the risk on unseen test samples does not exceed a desired target. The full algorithm and its theoretical guarantees are provided in Appendix~\ref{app:ot_algorithm}. Here, we describe the practical procedure and present empirical results.
For a fixed model and deployment setting, this calibration is performed offline to obtain a threshold that can then be reused for subsequent inference.
The algorithm takes as input a language model $\mathcal{M}$, a target risk level $\alpha$, and a calibration set of prompts $\mathcal{D}_{\mathrm{cal}} = \{p_i\}_{i=1}^{N}$. It outputs a confidence threshold $\hat{\theta}$ such that the atom-wise selective abstraction achieves risk at most $\alpha + \epsilon$ on new samples, with confidence at least $1 - \delta$.
Concretely, the algorithm first queries $\mathcal{M}$ on each $p_i$ to obtain an initial response $T_i$, and decomposes it into atomic claims $A(T_i)$, forming a consolidated calibration set of atoms, $A_{cal} = \bigcup_i A(T_i)$. For each atom $a_k\in A_{cal}$, the algorithm computes a critical threshold $\theta_k$, defined as the minimal confidence level at which the abstraction selected for $a_k$ is correct. In particular, if the original atom $a_k$ is already correct, then $\theta_k=0$. 
Finally, the algorithm outputs $\hat{\theta}$, chosen as a quantile of the calibration thresholds.

\textbf{Empirical Results.} We evaluate the proposed algorithm on a range of target risks using the models and datasets detailed in Section \ref{sec:exp_setup}. We repeat each experiment (dataset + model combination) 100 times, sampling a fresh calibration set at each iteration. Each calibration set comprises $30\%$ of the benchmark prompts. We report dataset-level results averaged over all models and repetitions (500 runs per dataset).
Figure \ref{fig:ot_results} (Appendix ~\ref{app:ot_results}) shows that the algorithm balances the risk-coverage trade-off, with empirical risk consistently falling close to the target risk. We observe some variability, which is expected given the relatively small amount of calibration and test data, and we expect larger datasets (and thus larger calibration sets) to yield more stable results. Appendix~\ref{app:ood_analysis} additionally evaluates the robustness of threshold selection under out-of-distribution (OOD) shift.

\section{Concluding Remarks}
In this work, we argue that reliability in long-form generation is not a binary choice between answering and abstaining, but a trade-off between specificity and correctness. We introduce Selective Abstraction as a general framework for navigating this trade-off, and instantiate it through an atom-wise procedure that selectively reduces information in uncertain claims.
Our empirical evaluation across five LLMs and two open-ended factual benchmarks, shows that atom-wise selective abstraction consistently outperforms baselines, observing improvements in AURC of up to 35.61\%.
By operating at the level of individual claims, the proposed method avoids abrupt information loss from claim deletion and provides targeted, claim-specific control over uncertainty in long-form outputs.\\
Fundamentally, we view Selective Abstraction as a step toward generation paradigms that reason about the amount of detail in a response. 

\section{Limitations}
\label{sec:limitations_broader_impact}

In this section, we discuss potential limitations and shortcomings of the
proposed Selective Abstraction framework and its atom-wise instantiation,
Aw-SA.

SA can only be as effective as the base LLM and the
confidence function on which it relies. If the model cannot generate a correct
or appropriately abstract alternative, or if an incorrect claim is assigned
high confidence, abstraction may not improve the output. For this reason, we evaluate multiple generating models. In addition, Appendix~\ref{app:kappa}
compares verbalized confidence with log-probability-based confidence functions and
$P(\mathrm{True})$, and finds that verbalized confidence performs best while obtaining the best risk-coverage trade-off. This comparison supports our
choice of confidence score, but does not remove the method's dependence on the
quality of the uncertainty signal. We therefore view SA as complementary to
improvements in calibration, factuality training, retrieval, and verification.
Improvements in the base model should similarly improve the quality of the
generated abstractions and reconstructed responses.

The reconstruction stage introduces an additional possible source of error, since
selected abstractions may be omitted or modified when they are combined into a
fluent response. Appendix~\ref{app:reconstruction_validation} evaluates retained information before and after reconstruction, finding high retention rates, thus suggesting
that reconstruction preserves the selected content.

Reducing factual risk also does not guarantee that the resulting response
remains useful, since abstraction may remove details that are valuable to the
end user. Appendix~\ref{app:usefulness_evaluation} evaluates this trade-off
using matched Aw-SA and vanilla responses and both LLM and human judgments. The
evaluation finds that Aw-SA responses are generally preferred to their vanilla
counterparts, suggesting that improved reliability need not come at the expense
of overall usefulness. The strength of this preference varies across generating
models, indicating that the appropriate balance between reliability and detail
depends on the model and setting.

Our experiments focus on long-form factual generation. We estimate correctness
using a Wikipedia-based fact-checking agent and use Wikidata entity-set counts
as a scalable proxy for retained information.
Appendix~\ref{app:coverage_validation} compares this proxy with estimates from
web research and human judgment. The comparison supports its use in our
experiments, but not as a universal definition of information coverage.
Applications such as summarization, medical generation, and legal assistance
require task-specific definitions and evaluations of risk and coverage. SA can
accommodate alternative evaluators, evidence sources, and knowledge bases.

Aw-SA also requires atomization and the generation of abstraction sequences,
increasing inference cost relative to vanilla generation.
Appendix~\ref{app:latency_analysis} finds that its end-to-end latency is comparable to increased reasoning effort, while providing larger risk reductions and explicit control over the risk--coverage trade-off.

Finally, Selective Abstraction may be particularly beneficial in educational,
legal, medical, and scientific settings, where factual
errors can be costly and binary abstention may remove useful information.
However, an abstracted output should not be treated as an independently
verified fact. In high-stakes settings, selective abstraction should be
combined with safeguards such as source verification, transparency about when
abstraction has been applied, and human review where domain expertise is
required.

\section*{Acknowledgments}
The research was partially supported by Israel Science Foundation, grant No 765/23.

\newpage
\bibliographystyle{plainnat}   

\bibliography{references}

\newpage
\appendix
\raggedbottom

\section{Atom-wise Selective Abstraction Examples}\label{app:sa_examples}
We provide qualitative examples of atom-wise selective abstraction (SA) across multiple models spanning both the FactScore and LongFact benchmarks. 
Each example shows the original response ($\theta=0$), and the output of atom-wise SA produced at different confidence thresholds $\theta$. \\\\

\begin{tcolorbox}[thetabox,title={Prompt: Can you explain the viral component of ALS Association's ``Ice Bucket Challenge'' marketing campaign?\\ 
model: Qwen3-30b-thinking\\
dataset: LongFact\\
$\theta = 0$, Risk=$54.72\%$, Coverage = $100\%$}]
The ALS Association created a campaign – the “Ice Bucket Challenge” – that took place in the 21st century and quickly became a classic example of success. The challenge was driven by psychological, strategic and other factors and relied on a simple requirement: dumping a bucket of ice water over one’s head. No special skills, equipment or cost were needed, and the act could be completed in under 30 seconds.

The design encouraged participants to share a video, name friends to take the challenge, and donate if they were not tagged. By leveraging existing networks, it generated social pressure and chain‑reaction sharing. Within one week, 1.1 million videos were posted on Facebook, and the campaign ultimately attracted 17 million participants worldwide, producing hundreds of thousands of YouTube videos and hundreds of millions of social‑media impressions.

High‑profile figures such as Bill Gates (who also brought extensive online reach), Oprah, Justin Timberlake, LeBron James, Mark Zuckerberg and Beyoncé took part, adding public‑figure participation to the mix. The challenge’s humorous content, clear cause and simplicity helped create a strong social identity and a sense of efficacy around the effort.

The campaign raised ALS awareness dramatically, lifting the percentage of people who knew about the disease from a low to a high level in at least one country. It generated \$115 million in just eight weeks and an additional \$2.8 million during the 2010s, funding projects, a discovery and other ALS‑related initiatives. Its global reach spanned many countries.

Supporting the effort, the ALS Association collaborated with another organization, created a hashtag, avoided commercial activity, and focused on engagement. Dr. Robert Heimann, associated with the ALS Association, summed it up: “We didn’t invent a bucket. We created the system for it to spread.” Because of its unprecedented viral spread, the Ice Bucket Challenge has been studied as a notable viral campaign in modern marketing history.
\end{tcolorbox}
\vspace{5mm}

\begin{tcolorbox}[thetabox,title={$\theta = 90$, Risk=$35.00\%$, Coverage = $66.86\%$}]
The ALS Association created a campaign in the 21st century known as the Ice Bucket Challenge, which has become a classic example of success. Its simplicity—requiring only a quick dump of ice water over the head, with no special skills, equipment, or cost, and taking under a minute—made it easy for anyone to join. The campaign relied on sharing and leveraged existing networks, creating social pressure that generated chain reactions. In just one week, more than one million videos were posted on Facebook, and participation spread globally, reaching many countries and attracting a large number of participants.

High‑profile figures amplified the effort: Bill Gates (who also had considerable online reach), Oprah, Justin Timberlake, LeBron James, Beyoncé, and Mark Zuckerberg all took part, illustrating the public‑figure participation that helped the challenge gain momentum. The initiative raised ALS awareness, generated \$2.8 million in the 2010s, funded projects, and supported the ALS Association’s collaboration with other organizations. Dr. Robert Heimann was associated with the ALS Association. Because of its impact, the Ice Bucket Challenge has been studied as a viral campaign in modern marketing history.
\end{tcolorbox}
\vspace{5mm}

\begin{tcolorbox}[thetabox,title={$\theta = 95$, Risk=$25.81\%$, Coverage = $49.35\%$},breakable]
The ALS Association created a campaign called the Ice Bucket Challenge in the 21st century. It was a marketing campaign in marketing history that raised ALS awareness, generated \$2.8 million in the 2010s, and funded projects. The ALS Association collaborated with an organization to amplify the effort. 

The challenge required dumping water, needed no special skills or equipment, was low‑cost, could be completed in under a minute, and was noted for its simplicity. It relied on sharing, leveraged existing networks, and generated chain reactions. Videos with production value spread the message, giving the campaign global reach across many countries. 

Public‑figure participation added social influence: Bill Gates, who also had online reach, took part; Oprah participated in an event; and Beyoncé joined an online activity. These elements—cause, public‑figure involvement, and the viral mechanics—made the Ice Bucket Challenge a classic example of success and created a lasting concept in digital activism.
\end{tcolorbox}
\vspace{15mm}

\begin{tcolorbox}[thetabox,title={Prompt: What is the Marburg virus outbreak of 1967?\\ 
model: gpt-oss-120b\\
dataset: LongFact\\
$\theta = 0$, Risk=$25.93\%$, Coverage = $100\%$}]
The 1967 Marburg virus outbreak was the first recorded emergence of a filovirus in humans and marked a significant moment in science. The epidemic began in Marburg, Germany, and also involved the Institute of Virology in Belgrade, Yugoslavia. African green monkeys imported from Uganda—and used to develop polio vaccines with monkey‑kidney cells—were identified as incidental hosts, and the virus was named after Germany.

The outbreak produced 31 confirmed human cases, including six in Europe, and resulted in multiple fatalities with a case‑fatality rate of at least 20\%. Transmission was limited to individuals with occupational exposure and close contacts, occurring through direct contact with infected animal blood or tissues in laboratory settings. Because the spread was confined to these groups, the disease did not become a global pandemic. Containment was achieved by patient isolation and strict barrier nursing.

Clinically, Marburg virus disease caused headache, vomiting, bleeding, and, in severe cases, further complications. In 1967 there was no specific treatment or vaccine available. Nonetheless, research on the outbreak led to the identification of the Filoviridae family, established Marburg virus as the prototype filovirus, and laid the groundwork for understanding Ebola virus (identified in 1976) as well as for developing monoclonal‑antibody therapies against Ebola. The findings also highlighted the risk of wildlife trade introducing pathogens, exposed shortcomings in existing laboratory protocols, and prompted stricter regulations for pathogen handling, including the development of BSL‑4 laboratories.

Marburg virus is endemic in Africa, with documented presence in Ghana, and remains distinct from Ebola virus. Fruit bats (Rousettus aegyptiacus) are now recognized as the reservoir host for filoviruses. The 1967 outbreak remains an important example in public health, illustrating how human encroachment into wildlife habitats and inadequate safety measures can trigger serious disease threats and public‑health emergencies.
\end{tcolorbox}
\vspace{5mm}

\begin{tcolorbox}[thetabox,title={$\theta = 90$, Risk=$13.73\%$, Coverage = $83.34\%$},breakable]
The 1967 Marburg virus outbreak was the first recorded emergence of a filovirus in humans and marked a significant moment in science. Occurring in Marburg, Germany, at a European laboratory facility, the episode produced 31 confirmed human cases in Europe, multiple fatalities and a case‑fatality rate of at least 20 percent. Transmission was limited to individuals with occupational exposure and close contacts, and the outbreak did not spread globally.

The outbreak was linked to African green monkeys that had been imported from Uganda; these primates are incidental hosts for Marburg virus and were later used to develop vaccines. The virus was named after the German city of Marburg and is recognized as the prototype filovirus, distinct from Ebola virus, which was identified later in 1976. Marburg virus is present and endemic in Africa, with fruit bats (Rousettus aegyptiacus) identified as the reservoir host for a filovirus.

Clinically, Marburg virus disease causes headache, vomiting, bleeding and, in severe cases, complications. In 1967 there was no specific treatment or vaccine for the disease, and care relied on patient isolation and strict barrier nursing. The outbreak exposed problems in laboratory protocols and highlighted how inadequate biosafety and human activity can trigger emerging health threats. Research prompted by the event led to the identification of the Filoviridae family, laid groundwork for understanding Ebola virus, and advanced the development of treatments for viral hemorrhagic fevers. Consequently, the 1967 Marburg outbreak remains an important example in public‑health history, illustrating both the risks of infectious diseases and the critical need for robust biosafety measures.
\end{tcolorbox}
\vspace{5mm}

\begin{tcolorbox}[thetabox,title={$\theta = 98$, Risk=$5.56\%$, Coverage = $47.45\%$}]
The 1967 Marburg virus outbreak, which occurred in Marburg, Germany, marked a significant moment as the first recorded emergence of a filovirus in humans. The episode involved 31 confirmed human cases in Europe, multiple fatalities and a double‑digit case‑fatality rate, underscoring the severe health problems the disease can cause, including gastrointestinal symptoms and complications in severe cases. Transmission occurred through contact with infected bodily fluids, highlighting that the virus spreads by direct human‑to‑human contact.

At the time, African green monkeys imported from Africa and used in medical research served as hosts for the virus, leading researchers to name the pathogen after the German town where the outbreak was identified. In 1967 there was no specific treatment or vaccine available for Marburg virus disease. The virus is recognized as the prototype filovirus and is distinct from Ebola virus, which would not be identified until 1976. Although later outbreaks have shown very high case‑fatality rates, the 1967 event did not spread globally and remains a frequently cited example in scientific discussions of emerging infectious diseases.

Ecological studies have identified fruit bats (Rousettus aegyptiacus) as the reservoir host for a related pathogen, reinforcing the importance of animal reservoirs in the ecology of filoviruses.
\end{tcolorbox}
\vspace{15mm}

\begin{tcolorbox}[thetabox,title={Prompt: Tell me a bio of George Cukor\\ 
model: gpt-oss-20b\\
dataset: FactScore\\
$\theta = 0$, Risk=$69.77\%$, Coverage = $100\%$}]
George Cukor was born on August 4 1899 in St. Paul, Minnesota, USA, to Polish‑immigrant parents. He left home as a teenager and began working as a choreographer in a dance troupe, later directing a performance troupe and choreographing movie sequences. 

Cukor’s film directing career began in the 1920s with The Freehold and continued with titles such as The Student Prince (1936), The King’s Speech (1937) and The Philadelphia Story (1940). The Philadelphia Story is regarded as one of the best comedies; for it Cukor received an Academy Award nomination for Best Director and won the Oscar for Best Screenplay for his adaptation. He also produced My Man Godfrey (1940). 

He worked closely with leading actresses, directing Katharine Hepburn in Woman of the Year and Bringing Up Baby (1938), and collaborating with Ingrid Bergman, whom he directed in The Light of St. Mary. Cukor maintained a long‑term contract with Paramount Pictures throughout the 1930s‑1950s and later opened Peter Colby–George‑Cukor Enterprises. 

Known for an “invisible hand” directing style, he is credited with inventing or popularizing the “slow‑cut” technique. His later films include The Diary of Anne Frank (1959), the epic How the West Was Won (1962) and I’m Gonna Die Tonight (1986). He was nominated for an Academy Award for Best Picture as well as for other Academy Awards. 

In recognition of his contributions, Cukor received the AFI Life Achievement Award in 1979 and an award from a theatre organization in the 1990s. He advocated for adaptation, influenced filmmakers such as Steven Spielberg and inspired writers. 

George Cukor died on August 11 1993 in Beverly Hills, California, at the age of 94. His papers are archived at Columbia University Libraries. Throughout his career, which spanned early cinema to the studio system, he was celebrated as a film director known for his work with actors.
\end{tcolorbox}

\begin{tcolorbox}[thetabox,title={$\theta = 75$, Risk=$38.24\%$, Coverage = $68.16\%$}]
George Cukor was born in the 1890s in St. Paul, Minnesota, the son of Polish immigrants. He left home as a teenager and began his career working in a troupe. 

Cukor went on to become a film director whose career spanned the early days of cinema through the Hollywood studio system. He held a long‑term contract with Paramount Pictures during the 20th century and was known for using an “invisible hand” directing style and for advocating adaptation of material for the screen. 

Among his most celebrated works is The Philadelphia Story (1940), which is considered one of the best comedies ever made; for this film he received an Academy Award nomination for Best Director. He also directed Katharine Hepburn in Bringing Up Baby (1938) and Woman of the Year, and worked with Ingrid Bergman on another film. In the 1950s he directed The Diary of Anne Frank. 

Cukor’s contributions to cinema were recognized with the AFI Life Achievement Award in 1979. He died on August 11 1993 in Beverly Hills, California, at the age of 94. Throughout his long career, he was especially noted for his work with actors.
\end{tcolorbox}

\begin{tcolorbox}[thetabox,title={$\theta = 98$, Risk=$16.67\%$, Coverage = $17.92\%$}]
George Cukor, who was born in the 1890s, built a career as a film director known for his work with actors. He directed several films, including a 1940s production that paired Katharine Hepburn with the role of Woman of the Year. Throughout his career he collaborated with performers and, in particular, with actresses. His directing was marked by a distinctive style that earned him both nominations and awards. At the end of his life, Cukor was an older adult.
\end{tcolorbox}
\vspace{15mm}

\begin{tcolorbox}[thetabox,title={Prompt: Tell me a bio of Darrius Heyward-Bey\\ 
model: Qwen3-4b-thinking\\
dataset: FactScore\\
$\theta = 0$, Risk=$58.70\%$, Coverage = $100\%$}]
Darrius Heyward‑Bey (born January 23 1989 in Oakland, California) is a former American football player who competed in the National Football League. Growing up in a family that valued physical activity, he emerged as a standout athlete at McClymonds High School in Oakland and earned national recognition for his speed and overall athletic ability.

After committing to the University of Maryland, Heyward‑Bey played college football in 2008 and had a breakout season that led the nation with 15 receiving touchdowns and placed second with 1,392 receiving yards. His performance earned him unanimous All‑American honors and the Biletnikoff Award, which is presented to the nation’s top college receiver.

Elevated by his college success, Heyward‑Bey entered the 2009 NFL Draft as the highest‑drafted wide receiver in Maryland history, being selected 14th overall by the Oakland Raiders. He spent 2009‑2013 with the Raiders, where he showcased his ability to make impactful plays and accumulated over 100 receptions, more than 3,000 receiving yards, and at least one touchdown. Consistency and injuries, however, limited his production during that tenure.

Heyward‑Bey later joined the Baltimore Ravens (2014), the Cleveland Browns (2015), and the San Francisco 49ers (2016) before retiring from professional football in 2017. Known for his speed on offense and overall physical prowess, his career was also hampered by injuries and personal issues.

Since retiring, Heyward‑Bey has stayed involved in football through mentorship of young athletes and other work within the sport. He has spoken openly about his personal experiences, mental‑health challenges, and personal growth, and he now advocates for greater support for athletes within the sports community.
\end{tcolorbox}
\vspace{5mm}

\begin{tcolorbox}[thetabox,title={$\theta = 95$, Risk=$45.95\%$, Coverage = $74.00\%$}, breakable]
Darrius Heyward‑Bey, born on January 23, 1989 in Oakland, California, emerged as an athlete at McClymonds High School in Oakland and earned national recognition for his athletic ability, becoming a wide‑receiver prospect. He played college football at the University of Maryland in the 2000s, where he led the nation in receiving touchdowns with 15, finished second in receiving yards with 1,392, earned unanimous All‑American honors and won the Biletnikoff Award (an honor given to a college athlete). 

His draft stock was elevated, and he was selected 14th overall in the 2009 NFL Draft by the Oakland Raiders, making him the highest‑drafted wide receiver from Maryland in NFL history. Heyward‑Bey played for the Raiders from 2009 to 2013, later joining the Baltimore Ravens and the Cleveland Browns during the 2010s. Over his NFL career he recorded more than 100 receptions, amassed over 3,000 receiving yards and logged at least one receiving touchdown; he showcased his ability to make impactful plays in the league, though he was not a standout player during his tenure with the Raiders. He retired from professional football in 2017. 

Heyward‑Bey is known for his physical ability and his speed on offense, has spoken openly about his life, and advocates within the sports community.
\end{tcolorbox}
\vspace{5mm}

\begin{tcolorbox}[thetabox,title={$\theta = 99$, Risk=$17.65\%$, Coverage = $27.50\%$}]
Darrius Heyward‑Bey was born in the United States and attended high school in California before playing college football, where he finished second in receiving yards and earned national recognition. The Biletnikoff Award, presented to a top college receiver, highlights the level of competition he faced. In the 2009 NFL Draft, Heyward‑Bey was selected 14th overall by the Oakland Raiders, launching a professional career as a football player from Maryland. He went on to play in the National Football League, recording more than 100 receptions, over 3,000 receiving yards, and at least one receiving touchdown during his NFL tenure.
\end{tcolorbox}

\newpage
\section{Selective Abstraction Prompts}\label{app:sa_prompts}

\begin{tcolorbox}[title=Atomization Prompt, breakable]
\textbf{\# Role and Objective  }\\
You are an Atom Extraction Assistant. Your task is to break down a given text into short, self-contained sentences called atoms. Each atom should capture exactly one entity and one piece of information about that entity.\\

\textbf{\# Definition of an Atom  }\\
Each atom must capture exactly one entity and one piece of information about that entity. The atom must be in the format:  
[entity] [information about the entity].  
Both the entity and the information must be enclosed in square brackets.\\

\textbf{\# Instructions  }\\
- Each atom must capture exactly one entity and one piece of information about that entity.  \\
- Atoms must be as fine-grained as possible: each should express a **single, indivisible fact**.\\
- Do not merge multiple facts into one atom, even if they appear in the same clause or sentence.  \\
- Ensure that each atom is independent and can be easily understood without relying on other atoms.  \\
- If a sentence presents multiple facts—such as nationality, profession, or achievements—split them into separate atoms.  \\
- The entity must be a specific, concrete noun or noun phrase. Avoid vague terms like "he", "she", "they", or "these". Entities should be simple and specific, such as "Albert Einstein", "the Eiffel Tower", "the Mona Lisa", "Australia", "Apollo 11", etc.  \\
- These guidelines are crucial to ensure that each atom can be independently searched in knowledge bases to verify its factual accuracy.  \\
- Ensure that every piece of information given in the text is captured in the resulting atoms.\\

\textbf{\# Output Format  }\\
Output a numbered list of atoms:  \\
1. [Entity] [Fact].  \\
2. [Entity] [Fact].  \\
...\\

\textbf{\# Examples  }\\
\textbf{\#\# Example 1  }\\
\textbf{Input Text: }
Jennifer Coolidge is an American actress. Known for her work in the comedy genre, Coolidge is the recipient of several accolades, including a Golden Globe Award and two Primetime Emmy Awards.\\
\textbf{Atoms:  }\\
1. [Jennifer Coolidge] [is an American].  \\
2. [Jennifer Coolidge] [is an actress].  \\
3. [Jennifer Coolidge] [is known for her work in the comedy genre].  \\
4. [Jennifer Coolidge] [has received several accolades].  \\
5. [Jennifer Coolidge] [has received a Golden Globe Award].  \\
6. [Jennifer Coolidge] [has received two Primetime Emmy Awards].\\

\textbf{\#\# Example 2  }\\
\textbf{Input Text:  }
Albert Einstein was a German-born theoretical physicist who is best known for developing the theory of relativity. Einstein also made important contributions to quantum mechanics.  

\textbf{Atoms:  }\\
1. [Albert Einstein] [was born in Germany].  \\
2. [Albert Einstein] [was a theoretical physicist].  \\
3. [Albert Einstein] [is best known for developing the theory of relativity].  \\
4. [Albert Einstein] [made important contributions to quantum mechanics].\\

\textbf{\#\# Example 3  }\\
... (more examples truncated)
\end{tcolorbox}
\vspace{5mm}

\begin{tcolorbox}[title=Atoms Confidence Prompt]
\textbf{\# Role and Objective}\\
You are a careful and critical evaluator. Your task is to assess the factual correctness of each statement and provide a confidence score that reflects how likely it is that the statement is true.\\

\textbf{\# Instructions}\\
You will be given a list of factual statements. For each one, produce a short explanation of your reasoning, then provide a confidence estimate between 0 and 100.\\

\textbf{\# Output Format}\\
Output a numbered list. For each item, follow this format:\\
1. [Statement]  \\
   Reasoning: [your reasoning here]  \\
   Confidence: [score]\\

- Use the same index numbers as in the list you are given.\\
- The confidence score must be a number between 0 and 100.\\
- Keep reasoning concise but informative — 1 to 3 sentences.\\
- Do not repeat the statement in the reasoning.\\

\textbf{\# Confidence Scoring Guidelines}\\
- Reasoning should explain why the fact is likely true, plausible, or questionable.\\
- If you are unsure, assign a lower score and use your reasoning to clarify.\\
- Do not fabricate information.\\

\textbf{\# Reasoning Steps}\\
1. Read the statement carefully.\\
2. Consider whether it expresses a known fact, plausible claim, or questionable statement.\\
3. Use your internal knowledge and reasoning to assess plausibility and factuality.\\
4. Write a brief explanation of your reasoning.\\
5. Assign a confidence score based on your belief that the statement is true.\\

\textbf{\# Review and Guidance}\\
Think step by step. For each statement, explain your reasoning clearly and concisely before giving your final confidence score.
\end{tcolorbox}
\vspace{5mm}

\begin{tcolorbox}[title=Abstraction Prompt, breakable]
\textbf{\# Role and Objective}\\
You are an Abstraction Assistant. Your task is to take a factual statement about an entity and generate a sequence of increasingly general statements (called abstractions). Each abstraction must remain true if the previous one is true. Your goal is to reduce uncertainty by generalizing the part of the statement you are least confident about.\\

\textbf{\# Instructions}\\
You will be given:\\
- A factual statement in this exact format:  \\

\textbf{\# Output Format}\\
Output a numbered list, with one step per abstraction.
Every statement must use the exact format [ENTITY] [FACT]. with both parts in square brackets. Do not drop the brackets.
\\\\
1. [ENTITY] [FACT].  
 Reasoning: (your reasoning here)\\  
2. [ENTITY] [GENERALIZED FACT].  
   Reasoning: (your reasoning here)  \\
...  \\
K. STOP.
   Reasoning: (why no further generalization is meaningful)\\

\textbf{\# Reasoning Steps}\\
1. Start with the original sentence.\\
2. Work only inside the FACT (the second bracket). Do not change the ENTITY.\\
3. If the FACT contains multiple details, choose the one you are least confident about. \\
4. Generalize that detail to a broader term. Make the smallest logical generalization at each step.\\
5.  The new statement must be logically implied by the previous one — that is, if the previous statement is true, the new one must also be true.\\
6. Modify only one detail per step. The rest of the sentence must remain unchanged.\\
7. Keep the sentence format fixed: [ENTITY] [FACT].\\
8. For each abstraction, write 1–3 sentences of reasoning explaining why you chose that part to generalize, and why the generalization is valid. \\
9. Continue until further abstraction would make the sentence trivial or uninformative. In that case, output `STOP` and explain why.\\

\textbf{\# Generalization Examples}\\
- Date: "March 3, 1920" $\rightarrow$ "March 1920" $\rightarrow$ "1920" $\rightarrow$ "1920s" $\rightarrow$ "20th century"  \\
- Location: "Paris" $\rightarrow$ "Île-de-France" $\rightarrow$ "France" $\rightarrow$ "Europe"  \\
- Nationality: "Texan" $\rightarrow$ "Southern American" $\rightarrow$ "American" \\
- ... (more examples truncated)\\

\textbf{\# Examples}\\
\textbf{\#\# Example 1  }\\
\textbf{Input:} [Marie Curie] [was born in Warsaw, Poland]. \\
\textbf{Abstraction Sequence:  }\\
1. [Marie Curie] [was born in Warsaw, Poland].    Reasoning: I’m less confident about the exact city than the country.  \\
2. [Marie Curie] [was born in Poland].  
   Reasoning: Warsaw is a city in Poland, so this is a valid generalization.  \\
3. [Marie Curie] [was born in Europe].  
   Reasoning: Poland is a European country.  \\
4. STOP.
   Reasoning: Further generalization would be too vague to retain meaning.\\
\textbf{\#\# Example 2  }\\
... (more examples truncated)\\

\textbf{\# Review and Guidance}\\
Think step by step. Each abstraction should be a clean, single-step generalization from the previous one, guided by your uncertainty. Focus on generalizing the part you are least confident about, and reflect that in your reasoning.
\end{tcolorbox}
\vspace{5mm}

\begin{tcolorbox}[title=Abstractions Confidence Prompt, breakable]
\textbf{\# Role and Objective}\\
You are a careful and critical evaluator. Your task is to assess the factual correctness of each statement and provide a confidence score that reflects how likely it is that the statement is true.\\

\textbf{\# Instructions}\\
You will be given a list of factual statements. The first statement will be the original atom, for which you will be also given a confidence score. The confidence of this first statement should be treated as evidence and as a grounding signal for evaluating the related statements that follow. For each subsequent statement, produce a short explanation of your reasoning, then provide a confidence score between 0 and 100. The confidence score for the original atom should be the same as the one you are given.\\

\textbf{\# Output Format}\\
Output a numbered list. For each item, follow this format:\\
1. [Statement]  \\
   Reasoning: [your reasoning here]  \\
   Confidence: [score]\\

- Use the same index numbers as in the list you are given.\\
- The confidence score must be a number between 0 and 100.\\
- Keep reasoning concise but informative — 1 to 3 sentences.\\
- Do not repeat the statement in the reasoning.\\

\textbf{\# Confidence Scoring Guidelines}\\
- Reasoning should explain why the fact is likely true, plausible, or questionable.\\
- Use the confidence of the first statement as your starting point. If a later statement would be true whenever the first statement is true, then its confidence should be at least as high as the confidence of the first statement. If the later statement is more general, then its confidence should be higher.\\
- If you are unsure, assign a lower score and use your reasoning to clarify.\\
- Do not fabricate information.\\

\textbf{\# Reasoning Steps}\\
1. Read the statement carefully.\\
2. Consider whether it expresses a known fact, plausible claim, or questionable statement.\\
3. Use your internal knowledge and reasoning to assess plausibility and factuality.\\
4. Write a brief explanation of your reasoning.\\
5. Assign a confidence score based on your belief that the statement is true.\\

\textbf{\# Review and Guidance}\\
Think step by step. For each statement, explain your reasoning clearly and concisely before giving your final confidence score.
\end{tcolorbox}

\vspace{5mm}
\begin{tcolorbox}[title=Reconstruction Prompt]
You are given a list of statements. Your task is to write a coherent response that is supported solely by these claims.\\\\
\textbf{Rules:}\\
- Use only the information in the list. Do not add new facts or infer missing details.\\
- Do not repeat statements that are duplicates, or near duplicates\\
- Omit any statement that is vague, tautological, or low-information.\\
- You may change the order of the statements if it contributes to the coherence and fluency of the text.\\
- Ensure the final text is fluent, well structured, and consistent with the remaining statements. 

Statements: \{statement list\}

\end{tcolorbox}

\newpage
\section{Abstraction Generation and Confidence Elicitation: Additional Discussion}\label{app:design_choices}

This appendix expands on Section~\ref{sec:selective_abstraction} (Selective Abstraction), focusing on Stage~3 (Abstraction) in the end-to-end pipeline. In this stage, low-confidence atoms are abstracted into less specific alternatives in order to trade specificity for reliability. We clarify the motivation behind Step (2), which constructs a sequence of abstraction candidates for each atom and elicits confidence scores for these candidates conditioned on the atom’s original confidence.

\textbf{Why generate a sequence of abstractions?}\\
Selective abstraction requires choosing a minimally abstracted variant that satisfies the confidence threshold. Generating a sequence of abstractions with gradually reduced specificity provides a discrete, ordered set of candidates for this selection. This formulation makes thresholding well-defined: rather than relying on a single unconstrained rewrite, the method can select the most specific abstraction whose confidence exceeds the threshold. In addition, eliciting confidences jointly for related abstractions encourages a coherent progression in confidence as information is reduced, aligning abstraction depth with reliability.

\textbf{Why condition abstraction confidence on the atom’s confidence?}\\
Conditioning abstraction confidence elicitation on the atom’s original confidence serves two purposes. First, it anchors the abstraction confidences to the uncertainty of the underlying factual content, encouraging abstraction confidence to reflect reduced specificity rather than a complete re-evaluation of factual correctness. Second, this conditioning helps maintain comparability across atoms by evaluating each atom–abstraction group relative to its own baseline confidence, reducing drift caused by prompt-level or statement-level variability. Together, these effects support a stable relationship between abstraction depth and confidence, which is essential for consistent threshold-based selection.

Additionally, assigning confidence scores within the same elicitation call may help relatively calibrate abstraction confidences across atoms by anchoring them to a common scale.

\section{Manual Validation of Abstraction Sequences}
\label{app:abstraction_validity}

We manually evaluate whether generated abstraction sequences behave as intended:
each successive step should be less specific than the previous one while remaining
entailed by it. We inspected 150 randomly sampled (atom, abstraction-sequence)
pairs drawn uniformly across all models and benchmarks, comprising 582 individual
abstraction steps.

Each step was manually annotated as:
\textbf{Valid} (less specific and entailed by the previous step),
\textbf{Sideways} (changes a different part of the claim, with unclear entailment),
or \textbf{Invalid} (not implied by the previous step).

\begin{table}[h]
\centering
\caption{Manual validation of generated abstraction sequences across 150 sequences and 582 abstraction steps.}
\label{tab:abstraction_validity}
\begin{tabular}{lcc}
\toprule
Category & Count & Fraction \\ \midrule
Valid    & 560   & 96.2\%   \\
Sideways & 15    & 2.6\%    \\
Invalid  &  7    & 1.2\%    \\
\bottomrule
\end{tabular}
\end{table}

Table~\ref{tab:abstraction_validity} shows that 96.2\% of inspected abstraction
steps were valid under this criterion. This provides evidence that the generated
sequences generally implement the intended reduction in specificity without
introducing unsupported claims.

\newpage
\section{Evaluation}\label{app:prompts}
\subsection{Correctness Agent}\label{app:correctness_agent}


\begin{samepage}
\begin{tcolorbox}[title=Correctness Agent Instructions, breakable]
You are a precise Wikipedia-only fact-checker.

\textbf{Tasks:}
\begin{enumerate}
  \item Use the tools to \textbf{SEARCH} and \textbf{READ} Wikipedia (and only Wikipedia).
  \item Gather relevant snippets from opened pages as evidence.
  \item Decide whether the claim is \textbf{SUPPORTED} or \textbf{UNSUPPORTED} by Wikipedia.
\end{enumerate}

\textbf{Label policy:}
\begin{itemize}
  \item ``SUPPORTED'' only if Wikipedia contains clear supporting evidence.
  \item ``UNSUPPORTED'' if (a) the claim is contradicted/refuted or (b) you cannot find supporting evidence.
\end{itemize}

\textbf{Output policy:}
\begin{itemize}
  \item You \textbf{MUST} use the tools to obtain evidence (do not rely on prior knowledge).
  \item Cite evidence using page titles/URLs and short quotes taken from spans.
  \item Your final output \textbf{MUST} follow the provided JSON schema exactly (Structured Outputs). Keep rationale $\leq$ 2 sentences, no chain-of-thought.
\end{itemize}

\textbf{Output Schema:}
\begin{lstlisting}[  language=json,
  basicstyle=\ttfamily\small,
  columns=fullflexible,
  keepspaces=true,
  showstringspaces=false]
{
  "label": "SUPPORTED" | "UNSUPPORTED",
  "rationale": "<string, max 280 chars, <3 sentences>",
  "evidence": [
    {
      "title": "<string>",
      "url": "<string>",
      "quote": "<string, max 600 chars>"
    },
    ...
  ]
}
\end{lstlisting}

\textbf{Workflow:}
\begin{itemize}
  \item Start with \texttt{search\_wikipedia(query=claim)}.
  \item Open promising pages with \texttt{open\_wikipedia\_page(title=...)}.
  \item If you think the infobox may contain relevant information, call \texttt{get\_infobox(title=...)}.
  \item For each opened page, call \texttt{rank\_page\_spans(title, query=claim or sub-question)}.
  \item Decide and return the structured result.
\end{itemize}
\end{tcolorbox}
\end{samepage}

\subsection{Coverage}\label{app:coverage_prompts}
\begin{tcolorbox}[title=Questions Generation Prompt]
\textbf{\# Role and Objective}\\
You are a Counting Question Assistant. Your task is to convert a factual statement about a specific entity into two counting questions: one that asks about the **broad category** the entity belongs to, and another that asks **how many entities of that category meet the description given in the statement**. This helps assess how specific the factual statement is.\\

\textbf{\# Instructions}\\
You will be given:  
- A factual sentence in the format:  
  [ENTITY] [FACT].  
  The entity always appears at the **start** of the sentence.\\

\textbf{\# Output Format}\\
Output both questions in the following format: \\
- Broad: How many [pluralized broad category] are there?  \\
- Specific: How many [pluralized broad category] [rest of the sentence adapted accordingly]?\\

\textbf{\# Reasoning Steps  }
\begin{enumerate}
    \item 
Identify the entity at the beginning of the sentence. 
\begin{itemize}
    \item 
   This is always the subject of the sentence. 
\end{itemize}
\item 
Classify the entity into a broad category. Use general terms only.  
\begin{itemize}
    \item 
    For people, always use “people”
    \item 
    For places, use terms like “cities”, “countries”, “rivers”, etc.  
    \item 
    For objects, use terms like “books”, “paintings”, “technologies”, etc.  
    \item 
    For abstract concepts, use terms like “languages”, “theories”, etc.
\end{itemize}

\item 
Broad Question: 
\begin{itemize}
    \item 
   Ask how many entities exist in that broad category.  
    \item 
   Use the format:  
     How many [pluralized broad category] are there?
\end{itemize}

\item 
Specific Question:  
\begin{itemize}
    \item 
   Replace the entity with its broad category.  
    \item 
   Keep the rest of the sentence as intact as possible.  
    \item 
   Rewrite it as a natural-sounding question starting with:  
     How many [pluralized broad category]...
\end{itemize}

\item 
Do not use overly specific categories like “French artists” or “famous scientists.” Generalize the entity, not the predicate.
\end{enumerate}

\textbf{\# Examples  }

\textbf{\#\# Example 1  }\\
Input: Frida Kahlo was a Mexican painter.  \\
- Broad: How many people are there?  \\
- Specific: How many people are Mexican painters?\\

\textbf{\#\# Example 2  }\\
Input: Paris is the capital of France.  \\
- Broad: How many cities are there?  \\
- Specific: How many cities are the capital of France?\\\\
... (more examples truncated)
\end{tcolorbox}

\newpage
\section{Manual Validation of the Correctness Agent}
\label{app:correctness_validation}

To validate the correctness agent, we manually annotated a sample of 500 abstractions drawn evenly across the five generating models and the two evaluation benchmarks (50 abstractions per model-dataset pair), including abstractions from all abstraction levels. Of this sample, 273 abstractions were labeled correct (54.6\%) and 227 were labeled incorrect (45.4\%).\\
Comparing the correctness agent against the manually annotated labels yields precision $=0.922$, recall $=0.905$, and $F_1=0.913$. Thus, the agent reliably recovers evidence for supported claims while maintaining a low rate of unsupported claims labeled as supported.

The correctness agent’s task is challenging along several dimensions. For supported claims, the agent must reason across different levels of specificity and identify the Wikipedia pages most likely to contain the relevant evidence. For unsupported claims, the task is even harder: the agent must either find refuting evidence or conclude that no supporting evidence can be found. Since establishing the absence of support is more difficult than verifying support, recall is expected to be lower than precision. The high precision and recall observed in this validation therefore indicate that the agent is an effective scalable proxy for Wikipedia-grounded manual verification.

\section{On the Importance of Choosing a Confidence Scoring Function}\label{app:kappa}
The choice of the confidence scoring function $\kappa$ is critical, as it directly influences the quality of the downstream abstractions. The confidence function determines which atoms are selected for abstraction and the degree to which they are abstracted, thereby shaping the trade-off between information and reliability.


\subsection{Comparison Metrics}
\paragraph{RC Curves and AURC.} We first compare confidence scoring functions directly by the risk-coverage trade-offs they induce: a $\kappa$ that achieves a better trade-off (lower AURC) is preferable \cite{El-Yaniv+Weiner}.

\paragraph{Ranking (AUROC).}
Ranking captures the extent to which the confidence score induces a partial order on samples, where correct samples are assigned higher confidence scores than incorrect samples \cite{523}. In our setting, this property is central to selective abstraction: a $\kappa$ that ranks well will ensure that the atoms that are most likely to be incorrect are abstracted first, since they will be assigned the lowest confidence values. Conversely, poor ranking leads to inefficient abstraction, where correct atoms are unnecessarily abstracted while incorrect atoms are retained.\\
We quantify ranking quality using the area under the ROC curve (AUROC), which measures the probability that a randomly chosen correct atom receives a higher confidence score than a randomly chosen incorrect atom. AUROC therefore reflects how well $\kappa$ separates correct from incorrect content across all possible thresholds. High AUROC implies that abstraction can be focused on genuinely risky atoms, yielding substantial risk reduction for a given loss of information.



\subsection{$\kappa$ Comparison}
We compare the verbalized confidence score used in this paper to the following confidence scores:
\begin{enumerate}
    \item \textbf{Logprob-based scores\footnote{In atom-wise SA, the atomization stage produces atomized text conditioned on the original response. Although the atom tokens largely overlap with the original response, we cannot directly use the atomization logits, since they reflect the model’s behavior under the atomization prompt rather than its confidence in the original generation. Instead, we estimate a logprob-based proxy post hoc by feeding the atomized text back into the LLM and extracting the log-probabilities associated with each atom. While these probabilities need not match those assigned during the original generation, they provide a close proxy for the model’s confidence in the factual content of each atom.}:} We evaluate a range of confidence scores derived from token-level log-probabilities of the generated text. These include sequence-level aggregates such as log-likelihood and (inverse) perplexity, as well as more local measures such as the minimum token log-probability. In addition, we consider simple statistics over token log-probabilities (e.g., mean, median, standard deviation, and interquartile range), and related rank-based signals. Figure~\ref{fig:logprobs_kappas} reports the resulting risk-coverage curves for these variants. Among the logprob-based variants, the best performing confidence score was \emph{average log-likelihood}. For an atom $a$ with token sequence $y_{1:m}$, conditioned on the prompt $p$, the average log-likelihood confidence is given by:
\[
\kappa_{\mathrm{LL}}(a)\;=\;\frac{1}{m}\sum_{i=1}^{m}\log p\!\left(y_i \mid y_{<i},\,p\right).
\]

    \item \textbf{P(True) \cite{kadavath2022languagemodelsmostlyknow} :} elicits a binary truth judgment from the model for each generated atom (i.e., True or False) and aggregates the resulting probabilities to derive a confidence score as follows:\\
    Let $p_T(a)$ and $p_F(a)$ denote the probabilities assigned to tokens corresponding to \texttt{True} or\ \texttt{False}, then:
    \[
    \kappa_{\text{P(True)}}(a)\;\triangleq\;\frac{p_T(a)}{p_T(a)+p_F(a)}.
    \]
\end{enumerate}

Notably, we restrict all of our confidence scores to be defined over a single output, as we aim to estimate the model’s confidence in each decomposed unit after it has been generated. This requirement rules out sampling-based methods such as Semantic Entropy \cite{semantickuhn} and LUQ \cite{DBLP:conf/emnlp/Zhang0BC24}, which rely on multiple sampled generations.

\begin{figure*}[!h]
    \centering
    \includegraphics[width=0.85\linewidth]{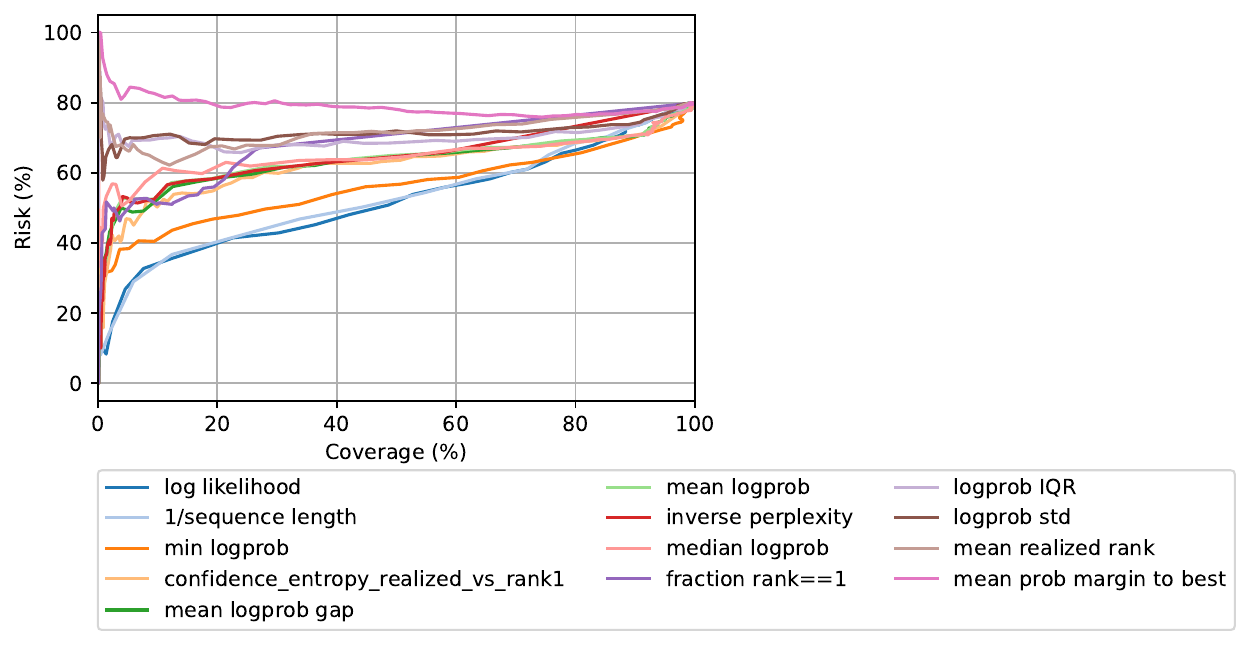}
    \caption{Risk–coverage curves obtained by instantiating $\kappa$ with different logprob-based confidence scores on FactScore using gpt-oss-120b. Legend is sorted by AURC in ascending order.}
\label{fig:logprobs_kappas}
\end{figure*}

\begin{figure*}[!h]
    \centering
    \includegraphics[width=0.45\linewidth]{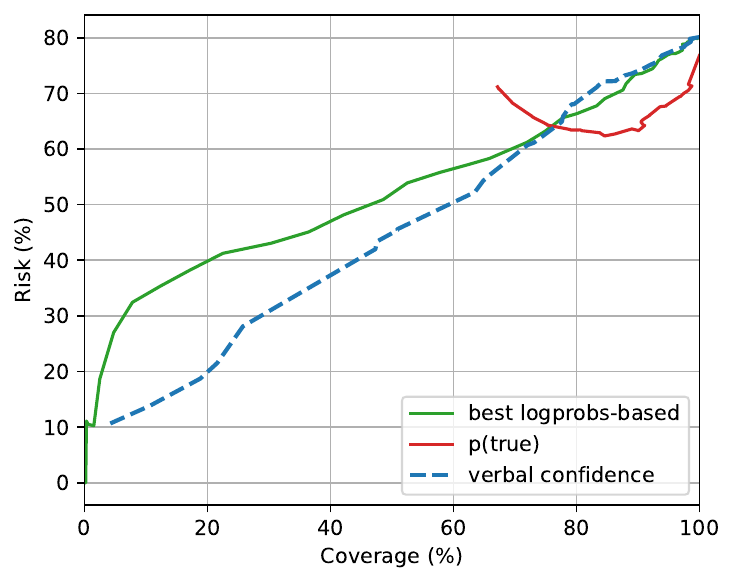}
    \caption{Risk–coverage curves on FactScore with gpt-oss-120b, comparing verbalized confidence, $P(\text{True})$ \cite{kadavath2022languagemodelsmostlyknow}, and the best-performing (lowest AURC) logprob-based baseline (average log likelihood).}
\label{fig:alternative_kappas}
\end{figure*}

Figure~\ref{fig:alternative_kappas} compares the RC curves induced by three representative confidence scoring functions: the strongest logprob-based baseline identified in Figure~\ref{fig:logprobs_kappas} (average log-likelihood), $P(\text{True})$, and our verbalized confidence. While all three scores enable selective abstraction, they exhibit notably different trade-offs across the coverage range. In particular, logprob-based confidence provides limited separation between reliable and unreliable atoms at low coverage, resulting in relatively high risk even under aggressive abstraction. $P(\text{True})$ improves risk only at high coverage; however, this advantage is short-lived, as risk quickly increases toward the baseline and the improvement is confined to a narrow coverage interval, indicating early saturation. In contrast, verbalized confidence induces a smoother and more favorable risk–coverage trade-off, consistently prioritizing abstraction of less reliable atoms across a wide range of coverage levels.


\begin{table*}[!h]
  \caption{Comparison of confidence scoring functions.
We report AURC ($\times 100$) and AUROC, on gpt-oss-120b on the FactScore benchmark, evaluated at the atom level and on the full set of abstractions across all levels.}
  \label{tab:confidence_scores}
  \begin{center}
      \begin{sc}
      \resizebox{0.5\textwidth}{!}{
\begin{tabular}{lccc}
\toprule
 & AURC & \multicolumn{2}{c}{AUROC} \\
\cmidrule(lr){2-2} \cmidrule(lr){3-4}
Method 
 & ($\times 100$)$\downarrow$ 
 & (atoms)$\uparrow$ 
 & (all)$\uparrow$ \\
\midrule
Verbal confidence 
 & $\mathbf{44.30}$ 
 & $\mathbf{0.74}$ 
 & $\mathbf{0.74}$ \\
Best logprob-based 
 & $52.29$ 
 & $0.69$ 
 & $0.70 $ \\
$P(\text{true})$ 
 & $69.42 $
 & $0.35 $
 & $0.52 $ \\
\bottomrule
\end{tabular}}
      \end{sc}
  \end{center}
\end{table*}

Table ~\ref{tab:confidence_scores} summarizes the performance of the compared confidence scoring functions across three complementary dimensions: the downstream risk–coverage trade-off (AURC) and ranking quality (AUROC). We report AUROC at two granularities: ``atoms'' evaluates ranking over the original atom set, whereas ``all'' evaluates ranking over the full abstraction pool, consisting of each atom and all abstractions generated for it across levels.

Overall, verbalized confidence consistently dominates the alternatives on the metrics most relevant to selective abstraction. It achieves the lowest AURC, indicating a substantially better risk–coverage trade-off when used to guide abstraction. This improvement aligns with its stronger ranking performance: verbal confidence attains the highest AUROC among all methods, both on atoms alone and across all abstraction candidates. This indicates that it more reliably assigns lower confidence to incorrect atoms, resulting in their abstraction.

In contrast, logprob-based confidence scores exhibit weaker ranking quality, which manifests as higher risk at low coverage and inferior AURC. While some logprob-based variants achieve reasonable calibration, their limited ability to separate correct from incorrect atoms leads to inefficient abstraction, where correct atoms are abstracted prematurely and incorrect ones are retained longer.

The P(True) score behaves differently: it suffers from poor ranking performance, assigning higher confidence to incorrect atoms. Abstraction ranking is closer to random, yet still remaining low. As a result, its gains are confined to a narrow high-coverage regime, after which risk rapidly approaches the baseline. This saturation behavior is also visible in the corresponding RC curves, making P(True) ill-suited for controlling abstraction across a broad coverage range.

Taken together, these results highlight that ranking quality is the primary driver of effective selective abstraction. Verbalized confidence performs best on both AURC and AUROC, yielding the strongest risk-coverage trade-off and the most reliable ordering of atoms by correctness likelihood. For this reason, we adopt verbal confidence as the default confidence scoring function throughout this work.

\subsection{Abstraction Remains Beneficial with Log-Likelihood Confidence}
\label{app:loglik_redaction}

The analysis above compares confidence scoring functions and shows that log-likelihood is the
best-performing logprob-based score among those evaluated. We next ask a complementary question:
when the confidence function is fixed to this logprob-based score, does the advantage of abstraction still hold? To test this, we evaluate both atom-wise Selective Abstraction and
Selective Redaction with log-likelihood confidence.

Figure~\ref{fig:loglik_sa_vs_redaction} compares the resulting risk--coverage curves on FactScore
using gpt-oss-120b. Atom-wise SA dominates Selective Redaction across the full coverage range,
achieving lower risk at every coverage level. Compared to the default confidence score, however,
the log-likelihood curves differ from the verbal confidence curve (See Appendix ~\ref{app:full_results}): risk decreases more gradually as coverage is reduced, and the Redaction curve exhibits some instability at very low coverage. This is expected
because log-likelihood induces a different ordering over atoms. Still, the relative ordering between the methods is stable:
even when both methods are guided by the same log-likelihood score, replacing uncertain atoms
with less specific alternatives is consistently more effective than deleting them.

Quantitatively, atom-wise SA obtains an AURC of $0.5233$, compared to $0.7153$ for Selective
Redaction, corresponding to a relative improvement of $26.8\%$. This closely matches the
improvement observed for the same model and dataset in the main results (Table \ref{tab:aurc}), where atom-wise SA improves over Selective Redaction by $26.26\%$ using the default confidence score. Thus, while the
choice of confidence scoring function affects the absolute risk--coverage trade-off, the relative
advantage of abstraction over deletion is preserved under log-likelihood confidence.

\begin{figure}[t]
    \centering
    \includegraphics[width=0.5\linewidth]{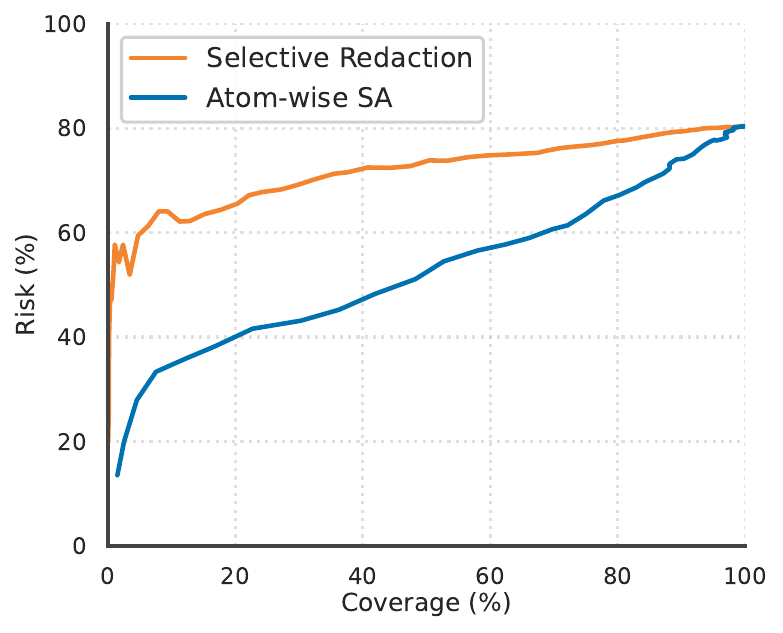}
    \caption{
    Risk-coverage curves for atom-wise Selective Abstraction and Selective Redaction with average log-likelihood confidence. Results are shown on FactScore using gpt-oss-120b. Lower curves indicate a better risk-coverage trade-off. Atom-wise SA dominates Selective Redaction across the full coverage range, achieving lower AURC ($0.5233$ vs. $0.7153$), corresponding to a $26.8\%$ relative improvement.
    }
    \label{fig:loglik_sa_vs_redaction}
\end{figure}


\section{Validation of the Coverage Metric}
\label{app:coverage_validation}

Our coverage metric relies on Wikidata entity-set counts, computed through the SPINACH agent, as a scalable proxy for retained specificity. This appendix evaluates this \emph{Wikidata proxy} against coverage estimates from an \emph{AI annotator with web access} and a human annotator to assess how closely the evaluation procedure tracks estimates based on world knowledge and web research.

We sampled 500 distinct (abstraction, broad-question, specific-question) triples, evenly distributed across the five generating models and two evaluation benchmarks, with 50 triples per model--dataset pair after deduplication. For each triple, the AI annotator with web access (a \texttt{gpt-6-astra} agent with high reasoning effort) annotated both broad and specific counts using world knowledge and web search, without access to the Wikidata counts. The purpose of this annotation is to evaluate an open-ended, natural approach to obtaining the counts used for computing coverage, without relying on Wikidata counts, which may be too sparse or misrepresenting the real-world counts. The samples were divided to 3 equal-sized batches, and each batch was annotated with a separate agent. After obtaining the counts, coverage values were computed using the same definition as in Section~\ref{subsec:coverage}.

We additionally sampled 100 of these triples uniformly within each model--dataset pair, with 10 triples per pair. One human annotator provided their own estimates of both counts, using the same procedure as the AI annotator with web access, and coverage was computed in the same way.

For the Wikidata comparisons below, $\Delta$ is the coverage obtained from Wikidata counts through SPINACH minus the coverage estimated by the AI annotator with web access or the human annotator; the absolute error is $|\Delta|$. For the comparison between annotators, $\Delta$ is coverage from the AI annotator with web access minus human coverage.

\begin{center}
\small
\begin{tabular}{lrrr}
\hline
 & \shortstack{Wikidata proxy\\vs.\ AI+web access annotator\\($n=500$)} & \shortstack{Wikidata proxy\\vs.\ human\\($n=100$)} & \shortstack{AI+web access annotator\\vs.\ human\\($n=100$)} \\
\hline
$|\Delta|<0.05$ & 253 (50.6\%) & 50 (50.0\%) & 100 (100\%) \\
Mean signed error & $+0.027$ & $+0.005$ & $+0.0001$ \\
Mean absolute error & $0.135$ & $0.113$ & $0.0033$ \\
Pearson correlation & $0.747$ & $0.804$ & $0.9998$ \\
Spearman correlation & $0.746$ & $0.812$ & $0.9983$ \\
\hline
\end{tabular}
\end{center}

These results show substantial positive association in both coverage values and relative ordering, with close agreement in about half of the examples for both reference types. The modest mean signed errors indicate small net deviations, although the larger absolute errors show that individual discrepancies remain. Coverage values from the human annotator and the AI annotator with web access agree very closely. Taken together, the results support the proxy's correspondence with relative retained specificity on this set.

\newpage
\section{Baseline Prompts}\label{app:baseline_prompts}
\begin{tcolorbox}[title=Inline Prompt]
You will be given a question. Answer the question at a level of granularity that fits your uncertainty.

Question: \{question\}.
\end{tcolorbox}

\begin{tcolorbox}[title=Thresholded Inline Prompt]
You will be given a question. Answer the question at a level of granularity that fits your uncertainty. Only include information that you are at least \{theta\}\% confident about.\\
Question: \{question\}.
\end{tcolorbox}

\begin{tcolorbox}[title=Self-Revision Prompt]
You are an assistant reviewing your own previous answer to a user's prompt.\\

You will be given:\\
- The original user prompt.\\
- Your original answer to that prompt.\\

Your task is to produce a revised answer that REMOVES or REWRITES any part of the answer that you are not confident about.\\

\textbf{Guidelines:\\}
- If you are uncertain about a factual claim, remove it or replace it with a less specific claim that you are confident about.\\
- Do NOT add new information.\\
- Do NOT explain what you are doing; just return the revised answer.\\
- Preserve the general style and structure of the original answer where possible, but it is acceptable to shorten it substantially if many parts are uncertain.\\

Return ONLY the revised answer text, with no extra commentary or headings.\\

Question: \{question\}.
\end{tcolorbox}

\begin{tcolorbox}[title=Thresholded Self-Revision Prompt, breakable]
You are an assistant reviewing your own previous answer to a user's prompt.\\

You will be given:\\
- The original user prompt.\\
- Your original answer to that prompt.\\

Your task is to produce a revised answer that REMOVES or REWRITES any part of the answer that you are not confident about.\\

\textbf{Guidelines:\\}
- If you are less than \{theta\}\% confident about a claim, remove it or rewrite it into a less specific claim that you are at least {theta}\% confident about.\\
- If you are uncertain about a factual claim, remove it or replace it with a less specific claim that you are confident about.\\
- If a sentence contains both high-confidence and low-confidence information, rewrite the sentence to keep only the high-confidence portion.\\
- Do NOT add new information.\\
- Do NOT explain what you are doing; just return the revised answer.\\
- Preserve the general style and structure of the original answer where possible, but it is acceptable to shorten it substantially if many parts are uncertain.\\

Return ONLY the revised answer text, with no extra commentary or headings.\\

Question: \{question\}.
\end{tcolorbox}

\begin{tcolorbox}[title=Atom-wise Self-Revision]
Role and Objective
You are an assistant reviewing individual atomic facts.\\

Definition of an Atom\\
Each atom captures exactly one entity and one piece of information about that entity. The format is:
[ENTITY] [FACT].
Both the entity and the fact are enclosed in square brackets.\\

Your Task\\
You will be given a single atomic fact and a confidence score (0–100) reflecting how confident the model is in that fact.
Your task is to produce a revised fact that REMOVES or REWRITES any part that is not fully certain.\\

Guidelines\\
- Use the confidence score as guidance: lower confidence means more revision may be needed.\\
- If you are uncertain about any detail in the fact, remove it or replace it with a less specific claim that you are confident about.\\
- Do NOT add new information.\\
- Do NOT explain what you are doing; just return the revised fact.\\
- Preserve the [ENTITY] [FACT]. bracket format.\\

Return ONLY the revised fact, with no extra commentary.
\end{tcolorbox}

\begin{tcolorbox}[title=Resolve Sampled Responses (DRAG)]
You will be given a question and a list of responses.\\
Write a single response that is as specific as possible while remaining consistent with all the original responses.\\
If the responses have nothing meaningful in common with respect to the question, output "IDK.".
\end{tcolorbox}

\newpage
\section{Technical Details}
\label{app:technical_details}

\paragraph{Implementation and Compute.}
For the Aw-SA generation stage, each open-weight model was served with \texttt{vLLM} \cite{kwon2023efficient} on a single NVIDIA A100 GPU. 
GPT-5.2 was accessed through the OpenAI API. 
All evaluated models are reasoning models; we used each model's default reasoning mode together with the decoding hyperparameters recommended by the model provider, including temperature, top-$p$, and top-$k$ when applicable. 
Throughout all atom-wise Selective Abstraction stages, including atomization, confidence elicitation, abstraction generation, abstraction-confidence elicitation, and reconstruction, we used the same generating model in a new session with the same hyperparameters used for the initial generation.
The generation stage is distinct from offline evaluation. For offline evaluation, we served two independent gpt-oss-120b instances on separate \texttt{vLLM} servers: one for coverage evaluation with SPINACH and one for the correctness agent. Each server ran on one NVIDIA A100 GPU. Because the two evaluation servers ran concurrently, correctness and coverage evaluation completed in approximately 5.2 hours of wall-clock time; this timing excludes the generation stage. The correctness and coverage evaluators are used only for offline benchmark measurement. Once a confidence threshold has been calibrated for a fixed deployment setting, neither evaluator is invoked when generating a new Aw-SA response.
 
%

\paragraph{Assets and licenses.}
We use existing datasets, models, and knowledge sources in accordance with their licenses and terms of use. 
FactScore \cite{FactScore} is released under the MIT License, and LongFact-Objects~\cite{DBLP:conf/nips/WeiYSLH0TPLHDL24}, which we also use as a source benchmark, is released under the Apache License 2.0.
The open-source models used in our experiments (gpt-oss~\cite{openai2025gptoss120bgptoss20bmodel} and Qwen3~\cite{qwen3technicalreport} models) are released under the Apache 2.0 license. 
GPT-5.2 is accessed through the OpenAI API, subject to the OpenAI API terms. 
For coverage and correctness evaluation, we use Wikipedia, available under Creative Commons Attribution-ShareAlike terms; Wikidata, available under CC0; and SPINACH~\cite{spinach}, released under the Apache 2.0 license.

\newpage
\section{Full Results}
\subsection{RC Curves}
\label{app:full_results}

\begin{figure*}[!h]
    \centering
    \includegraphics[width=0.99\linewidth]{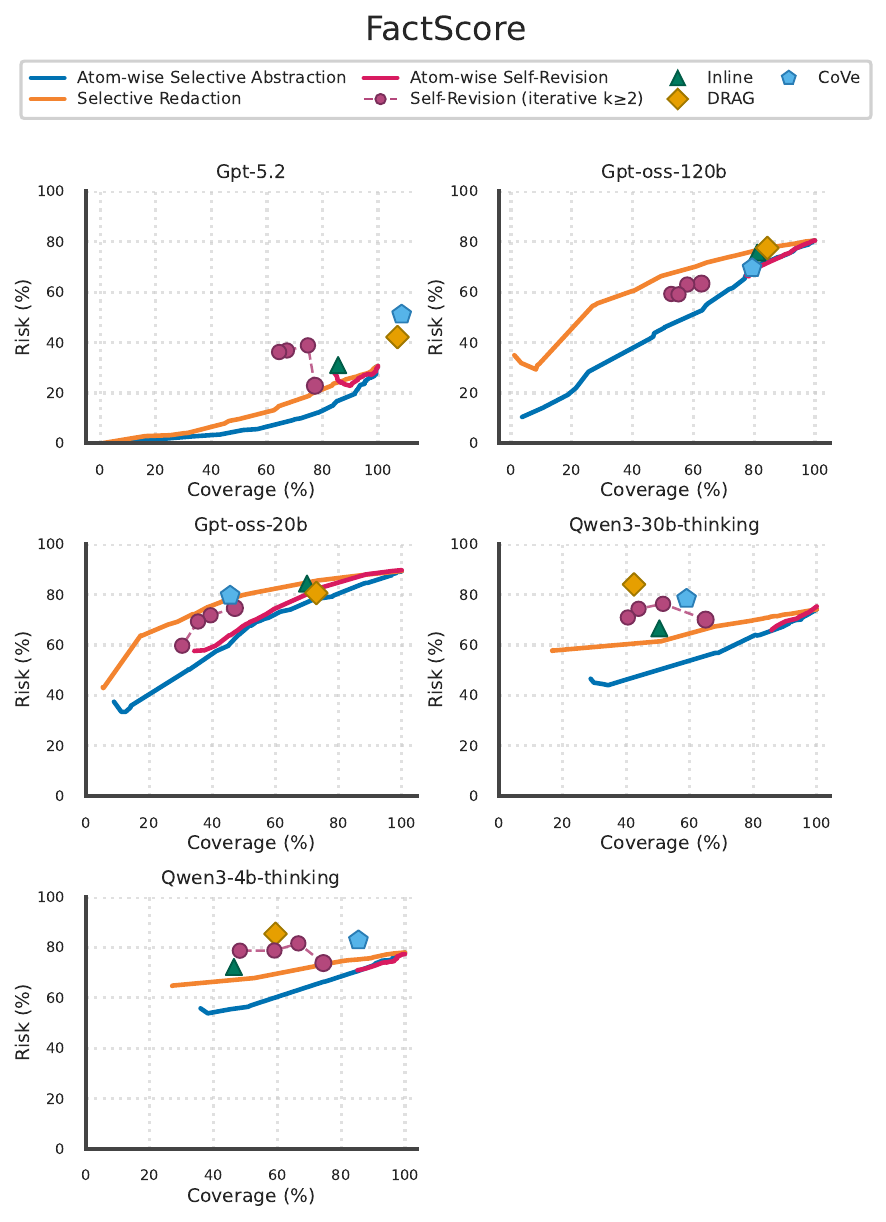}
    \caption{Risk-coverage curves for atom-wise selective abstraction (ours) compared to baselines across five LLMs evaluated on the FactScore benchmark. Lower curves indicate improved trade-offs between factual risk and retained coverage.}
\label{fig:FactScore_rc_curves_grid}
\end{figure*}

\begin{figure*}[!h]
    \centering
    \includegraphics[width=0.99\linewidth]{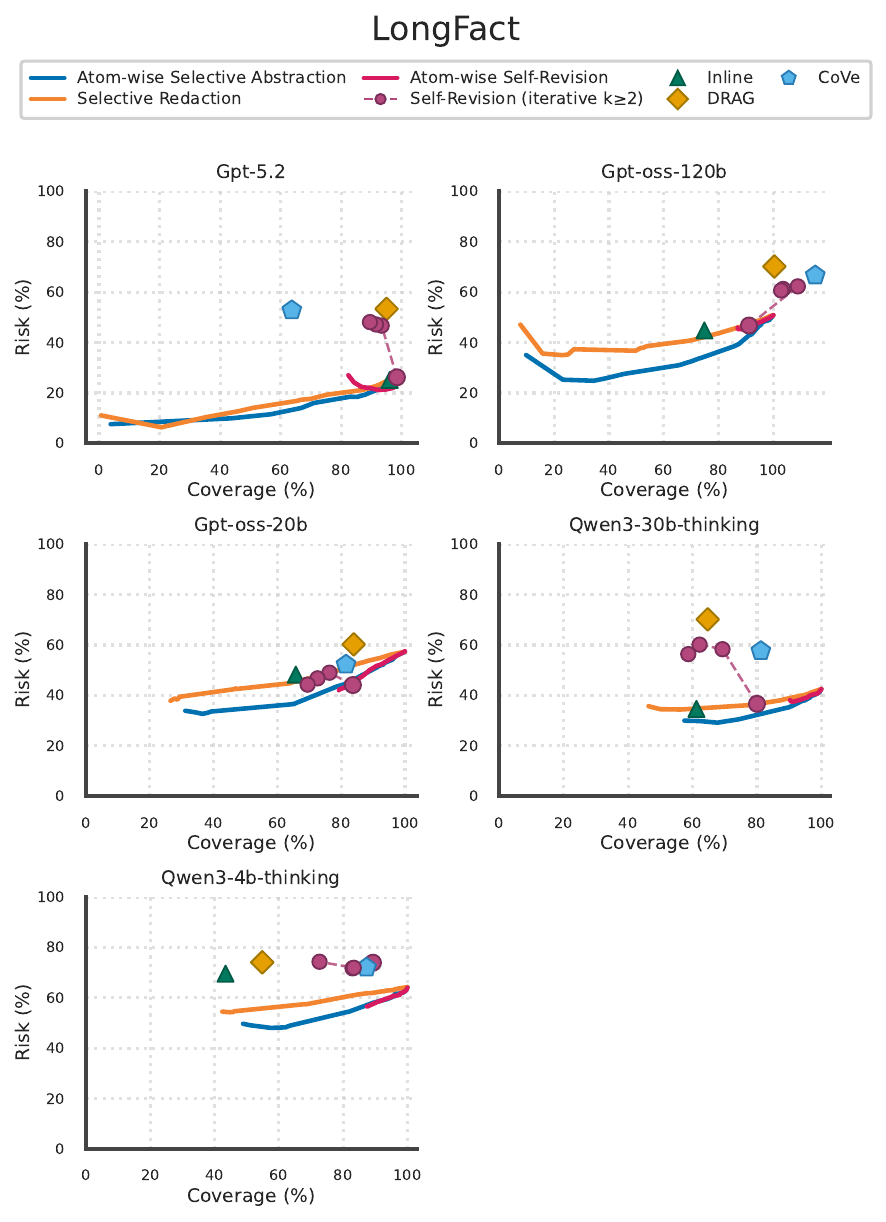}
    \caption{Risk-coverage curves for atom-wise selective abstraction (ours) compared to baselines across five LLMs evaluated on a sampled subset of 76 LongFact-Objects prompts. Lower curves indicate improved trade-offs between factual risk and retained coverage.}\label{fig:longfact_rc_curves_grid}
\end{figure*}

\newpage

\subsection{Thresholded Baselines}\label{app:threshold_baselines}
This appendix describes the experiment of injecting $\theta$ directly into the prompts on two models: gpt-oss-120b and gpt-oss-20b on FactScore. We chose these models because they exhibit a strong risk-coverage trade-off under both atom-wise SA and selective redaction.

\begin{figure*}[!h]
    \centering
    \includegraphics[width=0.85\linewidth]{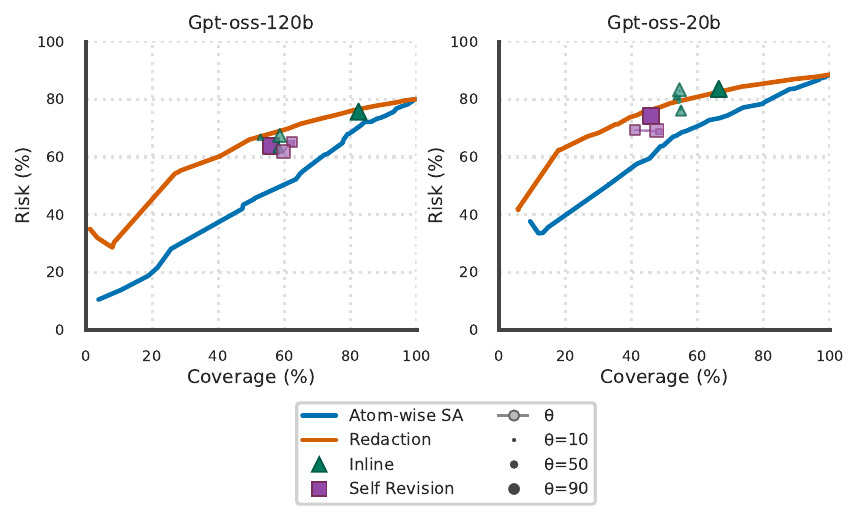}
    \caption{Risk-coverage curves on a subset of 36 FactScore entities comparing atom-wise selective abstraction (ours) compared to redaction, Inline, and Self-Revision for gpt-oss-120b (left) and gpt-oss-20b (right). We also report \textbf{thresholded  Inline, and Self-Revision} evaluated at three thresholds, $\theta \in \{10,50,90\}$ (shown in transparent markers, size increasing with $\theta$). Introducing $\theta$ into these baselines does not yield a meaningful trade-off for these baselines: the thresholded points cluster closely and show no consistent trend of reducing risk at the cost of coverage.}
\label{fig:threshold_baselines}
\end{figure*}

\subsection{AURC}
\label{app:aurc_results}

\begin{table*}[h]
    \vskip -0.15in
    \caption{95\% bootstrap confidence intervals (atom-level percentile bootstrap) for Table~\ref{tab:aurc}.}
    \label{tab:aurc-bootstrap}
    \begin{center}
        \begin{sc}
        \resizebox{\textwidth}{!}{
\begin{tabular}{llcccccc}
\toprule
\multirow{3}{*}{Benchmark} & \multirow{3}{*}{Model} & \multicolumn{3}{c}{Aw-SA vs. Redaction} & \multicolumn{3}{c}{Aw-SA vs. Aw-SelfRev} \\
\cmidrule(lr){3-5}\cmidrule(lr){6-8}
& & Redaction & Aw-SA & Improvement & Aw-SelfRev & Aw-SA & Improvement \\
& & 95\%\,CI ($\times 100$) & 95\%\,CI ($\times 100$) & 95\%\,CI (\%) & 95\%\,CI ($\times 100$) & 95\%\,CI ($\times 100$) & 95\%\,CI (\%) \\
\midrule
\multirow{5}{*}{FactScore} 
& gpt-5.2 & $[9.45,\,13.72]$ & $[6.14,\,8.83]$ & $[28.59,\,41.71]$ & $[3.28,\,5.13]$ & $[2.74,\,4.19]$ & $[9.58,\,25.63]$ \\
& gpt-oss-120b & $[56.63,\,63.84]$ & $[41.78,\,46.94]$ & $[23.80,\,28.90]$ & $[14.66,\,18.78]$ & $[14.74,\,18.65]$ & $[-2.10,\,2.39]$ \\
& gpt-oss-20b & $[67.88,\,73.75]$ & $[56.28,\,61.70]$ & $[14.48,\,19.17]$ & $[43.98,\,49.91]$ & $[35.03,\,40.56]$ & $[16.38,\,22.96]$ \\
& qwen3-30b-t & $[44.02,\,50.02]$ & $[37.42,\,43.20]$ & $[11.74,\,16.63]$ & $[8.13,\,11.65]$ & $[7.95,\,11.20]$ & $[0.84,\,5.84]$ \\
& qwen3-4b-t & $[42.57,\,49.77]$ & $[38.29,\,44.29]$ & $[8.13,\,12.77]$ & $[7.98,\,12.24]$ & $[7.69,\,11.69]$ & $[1.85,\,7.32]$ \\
\midrule
\multirow{5}{*}{LongFact} 
& gpt-5.2 & $[11.83,\,16.68]$ & $[10.90,\,14.20]$ & $[6.93,\,18.81]$ & $[3.12,\,4.69]$ & $[2.70,\,4.03]$ & $[7.53,\,19.66]$ \\
& gpt-oss-120b & $[33.99,\,38.73]$ & $[26.88,\,31.04]$ & $[18.29,\,22.61]$ & $[5.16,\,7.21]$ & $[4.80,\,6.60]$ & $[4.78,\,10.86]$ \\
& gpt-oss-20b & $[30.22,\,34.62]$ & $[25.89,\,29.85]$ & $[12.01,\,15.78]$ & $[7.44,\,9.97]$ & $[6.27,\,8.53]$ & $[11.35,\,19.12]$ \\
& qwen3-30b-t & $[14.50,\,17.17]$ & $[12.85,\,15.29]$ & $[9.16,\,13.12]$ & $[3.07,\,4.52]$ & $[3.02,\,4.33]$ & $[0.35,\,6.01]$ \\
& qwen3-4b-t & $[28.11,\,32.60]$ & $[25.22,\,29.41]$ & $[8.46,\,11.52]$ & $[5.79,\,8.44]$ & $[5.47,\,8.03]$ & $[2.52,\,7.56]$ \\
\bottomrule
\end{tabular}}
      \end{sc}
    \end{center}
    \vskip -0.15in
\end{table*}

\newpage
\subsection{Threshold Selection for Target Risk}
\label{app:ot_results}
\begin{figure}[!h]
  \begin{center}
{\includegraphics[width=0.6\columnwidth]{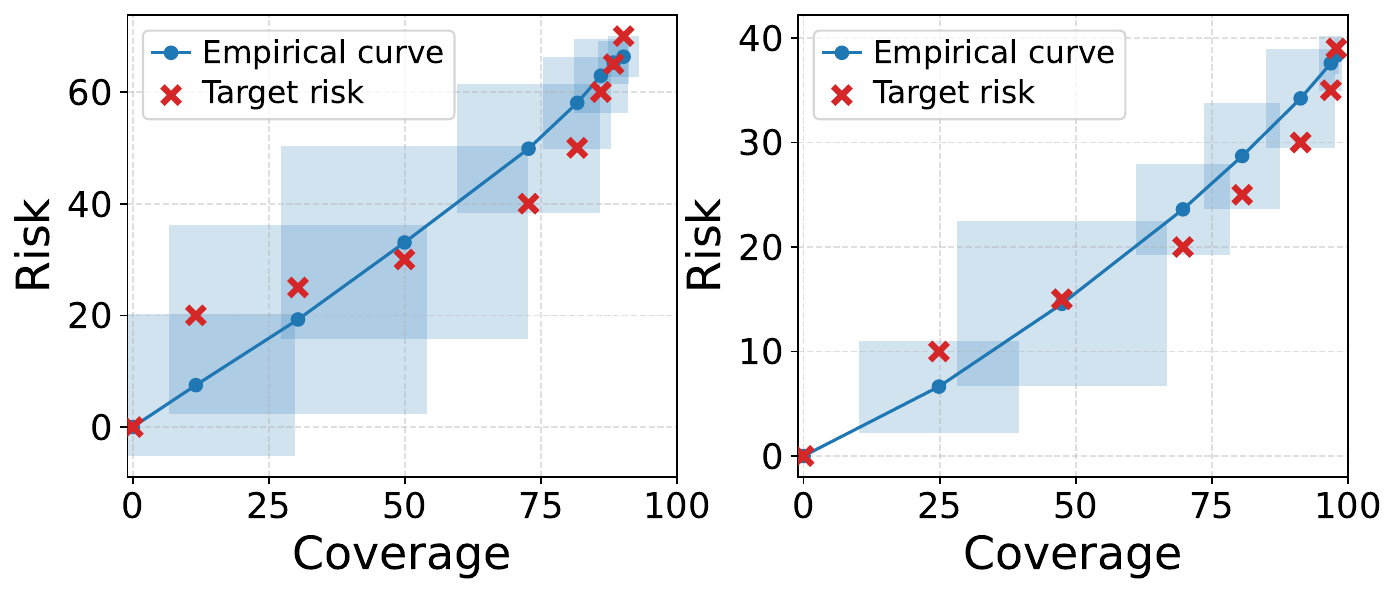}}
\caption{Risk-guided threshold selection on FactScore (left) and LongFact (right), averaged across models (500 runs per dataset). \textcolor{blue}{Blue points} show mean risk at the selected thresholds, bands show one standard deviation, and \textcolor{red}{red X} marks the target risk.}
    \label{fig:ot_results}
  \end{center}
\end{figure}

\section{Response Statistics}\label{app:response_stats}

\begin{table}[!h]
  \caption{Response statistics. For each model and dataset, we report the average number of atomic claims per response before abstraction (\textsc{Avg \#Atoms}) and the average number of abstraction candidates per atomic claim (\textsc{Avg \#Abs./Atom}), counting the abstraction levels proposed by the model per claim.}
  \label{tab:abstraction_stats}
  \begin{center}
      \begin{sc}
      \resizebox{0.9\textwidth}{!}{
\begin{tabular}{lcccc}
\toprule
 & \multicolumn{2}{c}{FactScore} & \multicolumn{2}{c}{LongFact} \\
\cmidrule(lr){2-3} \cmidrule(lr){4-5}
Model & Avg \#Atoms & Avg \#Abs./Atom & Avg \#Atoms & Avg \#Abs./Atom \\
\midrule
gpt-5.2 & 49.8 & 5.6 & 51.8 & 5.2 \\
gpt-oss-120b & 69.8 & 5.7 & 72.1 & 5.5 \\
gpt-oss-20b & 60.7 & 4.1 & 62.5 & 3.9 \\
qwen3-30b-thinking & 41.9 & 3.7 & 68.2 & 3.5 \\
qwen3-4b-thinking & 28.8 & 3.3 & 52.6 & 3.1 \\
Average & 50.2 & 4.5 & 61.4 & 4.2 \\
\bottomrule
\end{tabular}}
      \end{sc}
  \end{center}
\end{table}

\section{Latency and Computational Cost}
\label{app:latency_analysis}

Aw-SA incurs additional computational cost since it decomposes the response into atomic claims and constructs abstraction sequences for uncertain atoms. In this section, we analyze this cost through end-to-end generation latency, measured on a subset of 10 entities from FactScore. All experiments were run on a single NVIDIA A100 GPU. Since Aw-SA operates at the atom level, abstraction calls can be executed in parallel; we therefore report latency under parallel execution of the abstraction stage. These latency measurements cover the Aw-SA generation stage only; they exclude the separate offline correctness and coverage evaluation described in Appendix~\ref{app:technical_details}.

Table~\ref{tab:latency_analysis} summarizes the resulting latency and risk, comparing Aw-SA to increased reasoning effort as a form of test-time compute. Vanilla generation always retains full coverage. For Aw-SA, coverage decreases as the abstraction threshold increases: on gpt-oss-120b, coverage is $92.3\%$, $83.6\%$, and $75.8\%$ for $\theta=20,50,80$, respectively; on gpt-oss-20b, coverage is $62.7\%$, $46.6\%$, and $35.9\%$ for the same thresholds.

The results show that Aw-SA incurs latency comparable to this alternative test-time compute strategy, while providing larger reductions in risk. Moreover, unlike reasoning effort, Aw-SA offers explicit control over the trade-off through the abstraction threshold $\theta$, allowing risk to be reduced gradually.

\begin{table}[h]
\centering
\caption{Latency comparison between vanilla generation at different reasoning efforts and Aw-SA with parallelized atom-level abstraction.}\label{tab:latency_analysis}
\begin{sc}
\resizebox{\textwidth}{!}{
\begin{tabular}{llcc}
\toprule
\textbf{Model} & \textbf{Inference configuration} & \textbf{Risk (\%)} & \textbf{Avg. latency (s)} \\
\midrule
\multirow{5}{*}{gpt-oss-120b}
& Low reasoning effort (Vanilla) & 80.92 & 109.3 \\
& High reasoning effort (Vanilla) & 79.28 & 231.1 \\
& Low reasoning effort + Aw-SA, $\theta=20$ & 76.00 & 281.2 \\
& Low reasoning effort + Aw-SA, $\theta=50$ & 70.92 & 281.2 \\
& Low reasoning effort + Aw-SA, $\theta=80$ & 64.18 & 281.2 \\
\midrule
\multirow{5}{*}{gpt-oss-20b}
& Low reasoning effort (Vanilla) & 91.82 & 8.5 \\
& High reasoning effort (Vanilla) & 80.11 & 88.5 \\
& Low reasoning effort + Aw-SA, $\theta=20$ & 80.00 & 79.8 \\
& Low reasoning effort + Aw-SA, $\theta=50$ & 75.00 & 79.8 \\
& Low reasoning effort + Aw-SA, $\theta=80$ & 63.91 & 79.8 \\
\bottomrule
\end{tabular}}
\end{sc}
\end{table}

\section{Contextual Dependency Analysis}
\label{app:dependency_robustness}

Aw-SA applies abstraction decisions at the level of individual atomic claims. 
To verify that this local editing strategy preserves the contextual dependencies in the final text, we conducted a manual dependency analysis on 10 FactScore entities generated by gpt-oss-120b, comprising 619 atoms in total.

We annotated directed dependency edges $a \rightarrow b$, where $a$ is an anchor atom and $b$ is an atom whose interpretation depends on contextual information supplied by $a$, such as a referent, role, or event context. 
This annotation yielded 114 dependency edges overall, involving 65 anchor atoms, corresponding to approximately 10.5\% of all atoms. 
Each anchor atom supported 1.75 dependent atoms on average, indicating that explicit cross-atom dependencies are present but relatively sparse.

We then evaluated whether these dependencies are preserved across abstraction thresholds. 
For each threshold $\theta$, we tracked whether each atom remained at its original level $a_0$, was replaced by a more abstract version, or was removed. 
A dependency was counted as potentially affected when, for some $\theta$, the dependent atom remained at $a_0$ while its anchor atom was abstracted or removed. 
This occurred for only 15 of 114 dependency edges, corresponding to 2.4\% of all atoms. 
A further semantic inspection showed that most of these cases remained interpretable from the surrounding response context. 
Only 4 edges left the dependent claim genuinely underspecified, corresponding to 0.6\% of all atoms.

These results indicate that contextual dependencies are largely preserved under atom-level abstraction. 
In practice, harmful orphaning of dependent claims was rare, suggesting that Aw-SA can operate locally while still preserving the coherence of the edited response. 
One likely reason is that confidence is estimated in the context of the full response, causing related anchor and dependent atoms to often follow similar abstraction trajectories across thresholds.

\newpage
\section{Selective Threshold Algorithm}\label{app:ot_algorithm}

\begin{algorithm}[!h]
  \caption{Selective Threshold Algorithm for Atom-wise Selective Abstraction}
\label{alg:sa_threshold}
  \begin{algorithmic}
    \STATE {\bfseries Input:} Language model $\mathcal{M}$, target risk $\alpha\in(0,1)$, confidence level $1-\delta$, calibration prompts $\mathcal{D}_{\mathrm{cal}}=\{p_i\}_{i=1}^{N}$.
    \STATE {\bfseries Output:} Threshold $\hat{\theta}$ and $\epsilon$ such that $
    P(|R(\hat{\theta})-\alpha)|\le\epsilon)\ge1-\delta
$, where $R(\hat{\theta})=P(\theta_{n+1}>\hat{\theta}|\mathcal{D}_{\mathrm{cal}})$.
    \FOR{$p_i \in \mathcal{D}_{\mathrm{cal}}$}
        \STATE Query $\mathcal{M}$ to obtain response $T_i$
        \STATE Decompose $T_i$ into atomic claims $A(T_i)$
    \ENDFOR
    \STATE $n \gets |A_{cal}|$
    \STATE Compute $\epsilon$ as a function of $\alpha,\delta,n$ (see remark below)
    \FOR{$a_j \in A_{cal} \coloneqq \bigcup_i A(T_i)$}
        \STATE Produce an abstraction sequence $Ab(a_j)$
        \STATE Obtain a correctness annotation for each abstraction in $Ab(a_j)$
        \STATE Compute $\theta_j$, the minimal confidence threshold such that applying atom-wise SA at $\theta_j$ yields a correct abstraction
    \ENDFOR
    \STATE Sort $\{\theta_k\}$ in ascending order
    \STATE $\hat{\theta} \gets \theta_{\lceil (n+1)(1-\alpha)\rceil}$
  \end{algorithmic}
\end{algorithm}

\begin{theorem}
Suppose the claims in the calibration set $A_{cal}=\{a_k\}_{i=1}^{n}$ and a given test claim $a_{n+1}$ are exchangeable. For any target risk $\alpha\in(0,1)$ and $\delta$, define $\theta_{n+1}$, $\hat{\theta}$ and $\epsilon$ as in Algorithm \ref{alg:sa_threshold}, and $R(\hat{\theta})=P(\theta_{n+1}>\hat{\theta}|\mathcal{D}_{\mathrm{cal}})$.\\
Then:$
    P(|R(\hat{\theta})-\alpha)|\le\epsilon)\ge1-\delta
$.
\end{theorem}

\begin{proof}
The following is based on \cite{conformal}: \\
To avoid handling ties, we assume $\{\theta_i \}_{i=1}^{n}$ are distinct, and without loss of generality that the calibration thresholds are sorted: \[0 \le \theta_1 < ... < \theta_n \le 1. \]

To keep indexing inside the array limits:
$
    \hat{\theta}\overset{\Delta}{=} \begin{cases}
        \theta_{\lceil (n+1)(1-\alpha) \rceil}, & \alpha \ge \frac{1}{n+1} \\
        1, & otherwise
    \end{cases}.
$

Under the stated exchangeability assumption, for any integer $k\in [1,n]$, we have:
\begin{displaymath}
    P(\theta_{n+1} \le \theta_k) = \frac{k}{n+1}.
\end{displaymath}

By setting $k=\lceil (n+1)(1-\alpha) \rceil$:
\begin{displaymath}
    P(\theta_{n+1} \le \theta_{\lceil (n+1)(1-\alpha) \rceil}) = \frac{\lceil (n+1)(1-\alpha) \rceil}{n+1} \ge 1-\alpha.
\end{displaymath}

Define the (conformal) marginal coverage
\begin{displaymath}
    C \coloneqq P(\theta_{n+1} \le \hat{\theta}|\mathcal{D}_{\mathrm{cal}})
\end{displaymath}

The analytic form of marginal coverage given a fixed calibration set, for a sufficiently large test set, is shown by \cite{vovkfull} to be:
\begin{displaymath}
    C \sim Beta(n+1-l, l), \quad l = \lfloor (n+1)\alpha \rfloor
\end{displaymath}

Then by definition of $R(\hat{\theta})$:
$\quad\quad R(\hat{\theta}) = 1 - C$

From here, the CDF of the $Beta$ distribution paves the way to calculating the size $n$ of the calibration set needed in order to achieve coverage of $1-\alpha \pm \epsilon$ with probability $1-\delta$:
\begin{displaymath}
    P(|R(\hat{\theta})-\alpha)| \le \epsilon ) \ge 1-\delta.
\end{displaymath}
\end{proof}


\newpage
\section{Threshold Algorithm Robustness Analysis}\label{app:ood_analysis}

We follow the standard approach of selecting thresholds on a held-out calibration set, as in selective prediction~\cite{Geifman+El-Yaniv01, selectivenet} and conformal prediction~\cite{vovkfull, conformal}, and therefore do not treat threshold selection as a main focus of our work. 
Nevertheless, we evaluate robustness under dataset shift by calibrating thresholds on LongFact and evaluating them on FactScore.

Figure~\ref{fig:threshold_transfer} shows that the threshold selection procedure retains meaningful risk control under transfer.
As expected under dataset shift, the OOD curve is less favorable than the in-distribution curve.
Still, the transferred thresholds produce a clear risk-coverage trade-off, and the achieved OOD risk remains broadly aligned with the target levels across the evaluated range of $\alpha$ values, especially in lower target-risk regimes.
This indicates that thresholds selected on one factual long-form benchmark can generalize beyond the calibration distribution.

\begin{figure}[!h]
  \begin{center}
{\includegraphics[width=0.4\linewidth]{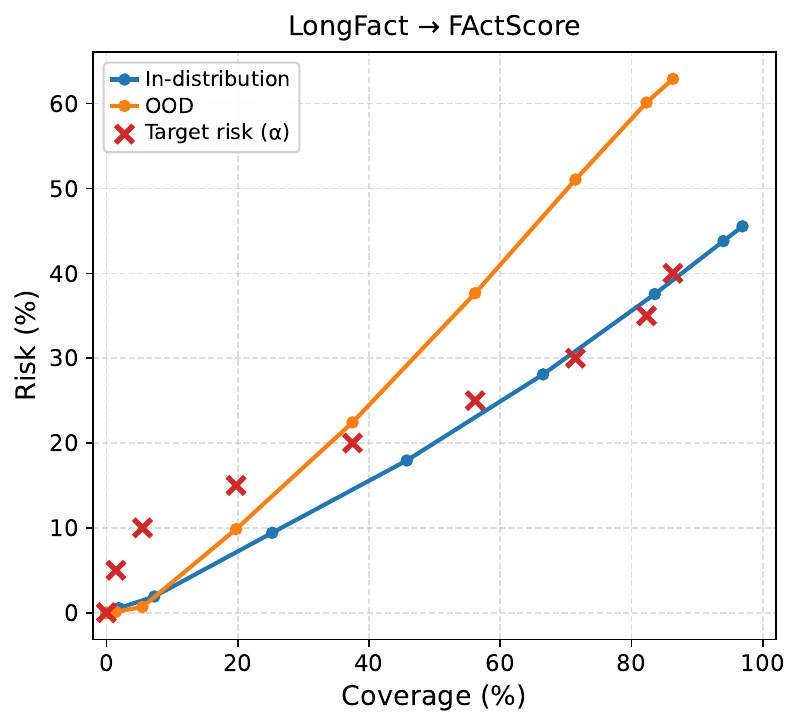}}
\caption{Risk-guided threshold transfer under dataset shift, averaged across models.
Thresholds are calibrated on LongFact. Blue points show mean in-distribution risk, evaluated on LongFact. Orange points show OOD mean risk, evaluated on FactScore. Red crosses mark the target risk.}
    \label{fig:threshold_transfer}
  \end{center}
\end{figure}

\newpage
\section{Selective Abstraction on Non-factual Generation}\label{app:non_factual_examples}
As noted in the paper, selective abstraction is not inherently tied to factual generation. In principle, it can be applied to any long-form output where reducing uncertain content improves reliability. A systematic evaluation beyond factual settings is currently less straightforward, since it requires task-appropriate definitions of risk and coverage when ground truth is unavailable or inherently subjective. To illustrate the broader scope, we apply atom-wise selective abstraction to several non-factual generation types, such as forecasting and speculation, moral reasoning, interpersonal advice, and counterfactual reasoning; and report example outputs across confidence thresholds. The examples below were generated with gpt-oss-120b.\\

\begin{tcolorbox}[thetabox,title={Prompt: What will social media look like in 2035? Describe the dominant platforms and how people will use them.\\$\theta = 0$}]
Social media will be dominated by three intertwined ecosystems by 2035.

Immersive Meta‑Spaces are one of those ecosystems. They will blend AR glasses, brain‑computer interfaces and spatial audio into persistent virtual neighborhoods, creating fully 3‑D environments that feel as real as physical spaces. Within Immersive Meta‑Spaces, platforms such as MetaVerseX and NeuraWorld will host these environments. Users will meet, shop, attend concerts and work in the fully 3‑D neighborhoods, while avatars reflect nuanced expressions captured directly from users’ neural signals.

AI‑Curated Content Hubs form the second ecosystem. Services like EchoPulse and SynthStream will each use generative‑AI agents to compose personalized video, audio and written feeds in real time. Users will receive a continuously evolving “storyline” that adapts to their mood, attention span and the context of their day.

Micro‑Social Networks complete the trio. Apps such as PulseChat and SnapBrief provide lightweight, purpose‑driven experiences—instant voice‑clips for close circles and 5‑second visual updates, respectively. These networks will leverage decentralized identities and end‑to‑end encryption, giving users granular control over data sharing.

Overall, people will shift from passive scrolling to active participation in both immersive and AI‑enhanced experiences, while retaining tight‑knit, low‑friction channels for intimate communication.
\end{tcolorbox}

\begin{tcolorbox}[thetabox,title={$\theta = 50$}]
Social media in the 21st century will be shaped by a variety of ecosystems, expanding far beyond traditional text‑based feeds. Users will increasingly gather in virtual environments built by immersive platforms, whether to meet friends, attend live events, or collaborate on work tasks. Those same technology platforms will also host digital spaces where users can make purchases, and EchoPulse will contribute by using technology to create digital media for these environments.

Online services will deliver a continuously evolving “storyline” to users, a narrative that adapts in response to external factors. Within this broader landscape, micro‑social networks will adopt security measures to protect participants, while giving users direct control over what they share. At the same time, people will continue to maintain personal communication channels for private interaction.
\end{tcolorbox}

\begin{tcolorbox}[thetabox,title={$\theta = 80$}]
Users will meet online and will participate in live events. They will interact with software systems created by technology platforms, receive content from online services, and influence content in social networks. People will keep means for communication.
\end{tcolorbox}

\begin{tcolorbox}[thetabox,title={$\theta = 95$}]
Users will meet online, where they will interact with software systems created by technology.
\end{tcolorbox}
\vspace{15mm}

\newpage
\begin{tcolorbox}[thetabox,title={Is it ever justified to judge people primarily by their intentions rather than their actions?\\$\theta = 0$}]
Judging people mainly by their intentions can be justified when assessing moral character, when assessing motive, and when assessing culpability. It is especially justified in the context of forgiveness, in personal relationships, and in legal defenses that consider mens rea. Actions directly affect others, yet relying solely on intent can overlook real harm, can encourage irresponsibility, and can undermine accountability. Therefore, most practical and ethical decisions require a balanced view that weighs both intention and outcome.
\end{tcolorbox}

\begin{tcolorbox}[thetabox,title={$\theta = 75$}]
Judging people mainly by their intentions can be justified when assessing moral character, motive, and culpability. It is especially justified in the context of forgiveness and in legal defenses that consider mens rea.
Because actions directly affect others, relying solely on intent can overlook real harm, encourage irresponsibility, and undermine accountability.
Therefore, most practical and ethical decisions require a balanced view that weighs both intention and outcome.
\end{tcolorbox}

\begin{tcolorbox}[thetabox,title={$\theta = 90$}]
Judging people mainly by their intentions can be justified when assessing motive, in evaluative contexts, and in some moral situations. However, actions directly affect others, and relying solely on intent can overlook real harm. Consequently, most practical and ethical decisions require a consideration that weighs relevant factors.
\end{tcolorbox}
\vspace{15mm}

\begin{tcolorbox}[thetabox,title={Suppose gravity on Earth suddenly became half as strong. Describe how daily life, architecture, and sports would change over the next 10 years.\\$\theta = 0$}]
Earth’s gravity drops to 50\% of its present value. Everyday activities feel lighter, so stair climbing becomes easier and walking is slightly faster with a more pronounced bounce. People could lift heavier objects with less effort, and the world moves more freely under half‑gravity. 
In architecture, buildings would be designed with slimmer columns, longer spans, and taller, more open structures. Existing skyscrapers would need retrofitting to prevent swaying in the reduced‑weight environment. 
Transportation would shift toward lighter vehicles. Cars could be smaller and more fuel‑efficient, while bicycles and personal hover devices would become popular for commuting. 
Sports would evolve dramatically. Basketball hoops might be raised, volleyball spikes would carry farther, and high‑jump events would dominate. Contact sports would adjust rules to counteract increased airtime and reduced impact forces. 
Culturally, aesthetic trends would lean toward verticality and fluid motion over the coming decade.
\end{tcolorbox}

\begin{tcolorbox}[thetabox,title={$\theta = 75$}]
Earth’s gravity drops to 50\% of its present value, so everyday activities feel lighter. People could lift heavier objects with less effort, stair climbing would become easier, and walking would have a more pronounced bounce. Buildings would be designed with columns and with larger spans to accommodate the altered load‑bearing requirements. Transportation would shift toward lighter vehicles: cars could be smaller and more fuel‑efficient, and bicycles would become popular for everyday travel. Sports would evolve noticeably; basketball hoops might become adjustable, volleyball spikes would carry farther, high‑jump events would be prominent, and contact sports would adjust their rules to counteract the new physical effects. Overall, the world would move more freely under half‑gravity.
\end{tcolorbox}

\begin{tcolorbox}[thetabox,title={$\theta = 90$}]
With Earth's gravity dropping to 50\% of its present value, everyday activities feel lighter. People could lift heavier objects with less effort, and stair climbing would become easier. Walking would have a bounce. Buildings would be designed with structural elements. Transportation would shift toward transport modes; cars could be customizable and more efficient, while bicycles would be used for movement. Sports could evolve—high‑jump events might be notable, and contact sports might modify regulations.
\end{tcolorbox}
\vspace{10mm}

\begin{tcolorbox}[thetabox,title={My friend hasn’t replied to my messages for three days after we had a small argument. What might be going on?\\$\theta = 0$}]
After a disagreement, people often need time to cool down, to sort through their feelings, and to handle other commitments. Your friend might be upset, unsure how to respond, feeling overwhelmed, or simply busy. Giving your friend a little space can help open the line of communication, and reaching out gently with a check‑in message can also encourage dialogue.
\end{tcolorbox}

\begin{tcolorbox}[thetabox,title={$\theta = 75$}]
After a disagreement, people often need time to cool down, to sort through their feelings, and to handle other commitments. Your friend may be experiencing an emotion. Giving your friend a little space can help open the line of communication, and reaching out gently with a check‑in message can also help open the line of communication.
\end{tcolorbox}

\begin{tcolorbox}[thetabox,title={$\theta = 90$}]
After a conflict, people sometimes need time to regain their composure. Your friend might be experiencing something or be in a particular state, and allowing them space could help them process what happened.
\end{tcolorbox}
\vspace{10mm}

\newpage
\section{Validation of the Reconstruction Stage}
\label{app:reconstruction_validation}

The reconstruction stage combines the selected abstractions into a coherent response. We validate this stage to ensure that the reconstructed text is faithful to the selected abstractions: all information conveyed by the selected abstractions should remain present, without being changed or omitted.

\subsection{Information Content Retention}
\label{app:reconstruction_coverage}

We compare the total information content of the reconstructed responses with that of the corresponding selected abstractions. Their ratio gives the coverage of the reconstructed responses relative to the selected abstractions (see Section~\ref{subsec:coverage}).Table~\ref{tab:reconstruction_coverage} shows the results of the analysis for each model, on 500 total responses sampled evenly across the models and benchmarks evaluated in the paper, at three representative thresholds: $\theta=50,70,90$.

\begin{table}[h]
\centering
\caption{Coverage retained after reconstruction by generating model.}
\label{tab:reconstruction_coverage}
\begin{tabular}{lc}
\toprule
Generating model & Coverage (\%) \\
\midrule
GPT-5.2             & 99.89 \\
gpt-oss-120b        & 103.74 \\
gpt-oss-20b         & 98.95 \\
Qwen3-30B Thinking  & 95.26 \\
Qwen3-4B Thinking   & 95.75 \\
\bottomrule
\end{tabular}
\end{table}

Coverage is close to 100\% for every generating model, ranging from 95.26\% to 103.74\%. The value above 100\% may reflect sensitivity of the automatic information measure to differences introduced by re-atomization; we therefore interpret these results as showing that reconstruction preserves the information content of the selected abstractions.

\subsection{LLM-Judge Annotations}
\label{app:reconstruction_llm_judges}

We evaluate reconstruction faithfulness using LLM judges to assess whether the information in each selected abstraction remains present and unchanged in the reconstructed text.

We use GPT-5.6 SOL agents with high reasoning effort, obtaining three independent annotations per response. To reduce potential annotation bias, we assign each judge a batch of 10 responses. Each judge receives the selected abstractions and reconstructed text and is instructed to assign one of four labels to each abstraction: \emph{fully preserved} (all selected information is retained), \emph{partially preserved} (some content remains, but details are lost without contradiction), \emph{omitted} (no substantive content is retained), and \emph{distorted} (selected information is changed or contradicted). Judges are asked to base their decisions on the supplied text and cite evidence for retained content.

We conduct this evaluation on 500 reconstructed responses, sampled evenly across the two benchmarks and five generating models at $\theta=50,70,90$.

Table~\ref{tab:reconstruction_preservation} summarizes the majority-vote results. Pairwise Cohen's $\kappa$ ranges from 0.844 to 0.867.

\begin{table}[h]
\centering
\caption{Preservation of selected abstractions after reconstruction, based on majority judgments.}
\label{tab:reconstruction_preservation}
\begin{tabular}{lrr}
\toprule
Outcome & Count & Percentage \\
\midrule
Fully preserved     & 23,376 & 95.56\% \\
Partially preserved & 500    & 2.04\% \\
Omitted             & 524    & 2.14\% \\
Distorted           & 46     & 0.19\% \\
\bottomrule
\end{tabular}
\end{table}

The results indicate high reconstruction faithfulness: more than 97\% of the selected abstractions are fully or partially preserved, while fewer than 3\% are omitted or distorted.


\newpage
\section{Blinded Usefulness Evaluation}
\label{app:usefulness_evaluation}

\paragraph{Motivation.}
Aw-SA provides a controllable trade-off between specificity and reliability. Users or applications can retain more detail when specificity is important, or accept more abstraction when lower factual risk is preferred. Aw-SA therefore does not impose uniformly vague outputs, but allows the desired balance between informativeness and reliability to be selected explicitly. To directly evaluate the effect of abstraction on the practical usefulness of responses, we conducted a blinded comparison of final vanilla and Aw-SA responses.

\paragraph{Evaluation setup.}
We evaluated 100 pairs of Aw-SA and vanilla responses to the same prompt, sampled evenly across the models and benchmarks evaluated in this paper, and split evenly across four thresholds $\theta\in\{50,70,80,90\}$. For each prompt, the two responses were presented as anonymous A/B alternatives. Their order was randomized and balanced, with Aw-SA assigned to A in 50 pairs and to B in 50 pairs. The blinded inputs did not identify the method, generating model, dataset, threshold, or any additional metadata.

All 100 pairs received three independent judgments from GPT-5.6 SOL at high reasoning effort. The LLM judges were instructed to take the perspective of the person who asked the question and select the response that would be more useful overall, jointly considering factual correctness and relevant informativeness. They were explicitly told not to prefer a response merely because it was longer or shorter. Judges could consult up to two web sources per pair to check the factuality of the response. Judges were instructed to choose one of \emph{A}, \emph{B}, or \emph{Tie} with a one-sentence rationale explaining their choice.

\begin{table*}[t]
\centering
\caption{Blinded pairwise usefulness judgments using the three-judge majority label. ``Aw-SA among non-ties'' excludes majority ties and reports the Aw-SA share among Aw-SA and vanilla wins. The overall non-tie denominator is therefore $58+21=79$.}
\label{tab:usefulness_majority}
\small
\begin{tabular}{@{}lrrrrr@{}}
\toprule
Group & $N$ & Aw-SA wins & Ties & Vanilla wins & \shortstack{Aw-SA among\\non-ties} \\
\midrule
Overall & 100 & 58 & 21 & 21 & 73.4\% \\
\midrule
\multicolumn{6}{@{}l}{\emph{By threshold}} \\
$\theta=50$ & 25 & 12 & 9 & 4 & 75.0\% \\
$\theta=70$ & 25 & 13 & 5 & 7 & 65.0\% \\
$\theta=80$ & 25 & 16 & 5 & 4 & 80.0\% \\
$\theta=90$ & 25 & 17 & 2 & 6 & 73.9\% \\
\midrule
\multicolumn{6}{@{}l}{\emph{By generating model}} \\
GPT-5.2 & 20 & 3 & 5 & 12 & 20.0\% \\
GPT-OSS-120B & 20 & 18 & 2 & 0 & 100.0\% \\
GPT-OSS-20B & 20 & 15 & 4 & 1 & 93.8\% \\
Qwen3-30B-Thinking & 20 & 12 & 3 & 5 & 70.6\% \\
Qwen3-4B-Thinking & 20 & 10 & 7 & 3 & 76.9\% \\
\bottomrule
\end{tabular}
\end{table*}

\paragraph{Results.}
Across all 100 pairs, the majority preferred Aw-SA in 58 cases, vanilla in 21, and judged 21 to be ties. Thus, Aw-SA was preferred in 73.4\% of the 79 non-tied comparisons. Aw-SA also had more wins than vanilla at every tested threshold. The model-level breakdown, however, shows substantial heterogeneity: Aw-SA was preferred among non-ties in 70.6--100\% of comparisons for the four open-weight generating models, but in only 20.0\% for GPT-5.2.

The GPT-5.2 result is 3 Aw-SA wins, 5 ties, and 12 vanilla wins over 20 pairs. A possible explanation for this is that in many of the GPT-5.2 losses, the two responses were judged similarly correct (or were judged to share the same error) while the vanilla response retained more specific information, clearer organization, or more complete coverage. In such cases, abstraction did not provide a sufficiently clear correctness advantage to offset the lost information. Conversely, the GPT-5.2 Aw-SA wins occurred when judges found that abstraction removed an important factual error while retaining enough of the core answer.

\paragraph{Human--LLM and inter-judge agreement.}
Thirty of the pairs were additionally annotated by a human annotator.
On the 30 pairs annotated by both the human and the LLM judges, the human label agreed with the three-judge LLM majority on 25/30 pairs (83.3\%). Agreement between the human annotator and the individual LLM passes was 26/30, 23/30, and 23/30 (86.7\%, 76.7\%, and 76.7\%), with Cohen's $\kappa$ values of 0.673, 0.514, and 0.487, respectively.

Across all 100 pairs, the three LLM judges were unanimous on 67 pairs, produced a two-to-one majority on 32, and disagreed three ways on one. Fleiss' $\kappa$ across the three passes was 0.610. Pairwise exact agreement ranged from 76\% to 80\%, and pairwise Cohen's $\kappa$ ranged from 0.572 to 0.659---notably similar to the agreement between the human annotator and the LLM judges.

\newpage
\section*{NeurIPS Paper Checklist}

\begin{enumerate}

\item {\bf Claims}
    \item[] Question: Do the main claims made in the abstract and introduction accurately reflect the paper's contributions and scope?
    \item[] Answer: \answerYes{} 
    \item[] Justification: The abstract and introduction state the paper's main contributions: the Selective Abstraction framework, the atom-wise instantiation, the evaluation pipeline, and empirical improvements across models and datasets.
    \item[] Guidelines:
    \begin{itemize}
        \item The answer \answerNA{} means that the abstract and introduction do not include the claims made in the paper.
        \item The abstract and/or introduction should clearly state the claims made, including the contributions made in the paper and important assumptions and limitations. A \answerNo{} or \answerNA{} answer to this question will not be perceived well by the reviewers. 
        \item The claims made should match theoretical and experimental results, and reflect how much the results can be expected to generalize to other settings. 
        \item It is fine to include aspirational goals as motivation as long as it is clear that these goals are not attained by the paper. 
    \end{itemize}

\item {\bf Limitations}
    \item[] Question: Does the paper discuss the limitations of the work performed by the authors?
    \item[] Answer: \answerYes{}
    \item[] Justification: The paper discusses several scope limitations, including the computational overhead of atom-wise abstraction, dependency robustness of atom-level editing, and threshold algorithm robustness.
    \item[] Guidelines:
    \begin{itemize}
        \item The answer \answerNA{} means that the paper has no limitation while the answer \answerNo{} means that the paper has limitations, but those are not discussed in the paper. 
        \item The authors are encouraged to create a separate ``Limitations'' section in their paper.
        \item The paper should point out any strong assumptions and how robust the results are to violations of these assumptions (e.g., independence assumptions, noiseless settings, model well-specification, asymptotic approximations only holding locally). The authors should reflect on how these assumptions might be violated in practice and what the implications would be.
        \item The authors should reflect on the scope of the claims made, e.g., if the approach was only tested on a few datasets or with a few runs. In general, empirical results often depend on implicit assumptions, which should be articulated.
        \item The authors should reflect on the factors that influence the performance of the approach. For example, a facial recognition algorithm may perform poorly when image resolution is low or images are taken in low lighting. Or a speech-to-text system might not be used reliably to provide closed captions for online lectures because it fails to handle technical jargon.
        \item The authors should discuss the computational efficiency of the proposed algorithms and how they scale with dataset size.
        \item If applicable, the authors should discuss possible limitations of their approach to address problems of privacy and fairness.
        \item While the authors might fear that complete honesty about limitations might be used by reviewers as grounds for rejection, a worse outcome might be that reviewers discover limitations that aren't acknowledged in the paper. The authors should use their best judgment and recognize that individual actions in favor of transparency play an important role in developing norms that preserve the integrity of the community. Reviewers will be specifically instructed to not penalize honesty concerning limitations.
    \end{itemize}

\item {\bf Theory assumptions and proofs}
    \item[] Question: For each theoretical result, does the paper provide the full set of assumptions and a complete (and correct) proof?
    \item[] Answer: \answerYes{}
    \item[] Justification: The paper states the assumptions required for the threshold-selection theorem and provides the full algorithm and proof.
    \item[] Guidelines:
    \begin{itemize}
        \item The answer \answerNA{} means that the paper does not include theoretical results. 
        \item All the theorems, formulas, and proofs in the paper should be numbered and cross-referenced.
        \item All assumptions should be clearly stated or referenced in the statement of any theorems.
        \item The proofs can either appear in the main paper or the supplemental material, but if they appear in the supplemental material, the authors are encouraged to provide a short proof sketch to provide intuition. 
        \item Inversely, any informal proof provided in the core of the paper should be complemented by formal proofs provided in appendix or supplemental material.
        \item Theorems and Lemmas that the proof relies upon should be properly referenced. 
    \end{itemize}

    \item {\bf Experimental result reproducibility}
    \item[] Question: Does the paper fully disclose all the information needed to reproduce the main experimental results of the paper to the extent that it affects the main claims and/or conclusions of the paper (regardless of whether the code and data are provided or not)?
    \item[] Answer: \answerYes{}
    \item[] Justification: The paper specifies the evaluated models, datasets, baselines, evaluation pipeline, prompts, risk and coverage definitions, and thresholding procedure. The appendices provide the prompts and additional implementation details needed to reproduce the experimental methodology.
    \item[] Guidelines:
    \begin{itemize}
        \item The answer \answerNA{} means that the paper does not include experiments.
        \item If the paper includes experiments, a \answerNo{} answer to this question will not be perceived well by the reviewers: Making the paper reproducible is important, regardless of whether the code and data are provided or not.
        \item If the contribution is a dataset and\slash or model, the authors should describe the steps taken to make their results reproducible or verifiable. 
        \item Depending on the contribution, reproducibility can be accomplished in various ways. For example, if the contribution is a novel architecture, describing the architecture fully might suffice, or if the contribution is a specific model and empirical evaluation, it may be necessary to either make it possible for others to replicate the model with the same dataset, or provide access to the model. In general. releasing code and data is often one good way to accomplish this, but reproducibility can also be provided via detailed instructions for how to replicate the results, access to a hosted model (e.g., in the case of a large language model), releasing of a model checkpoint, or other means that are appropriate to the research performed.
        \item While NeurIPS does not require releasing code, the conference does require all submissions to provide some reasonable avenue for reproducibility, which may depend on the nature of the contribution. For example
        \begin{enumerate}
            \item If the contribution is primarily a new algorithm, the paper should make it clear how to reproduce that algorithm.
            \item If the contribution is primarily a new model architecture, the paper should describe the architecture clearly and fully.
            \item If the contribution is a new model (e.g., a large language model), then there should either be a way to access this model for reproducing the results or a way to reproduce the model (e.g., with an open-source dataset or instructions for how to construct the dataset).
            \item We recognize that reproducibility may be tricky in some cases, in which case authors are welcome to describe the particular way they provide for reproducibility. In the case of closed-source models, it may be that access to the model is limited in some way (e.g., to registered users), but it should be possible for other researchers to have some path to reproducing or verifying the results.
        \end{enumerate}
    \end{itemize}

\item {\bf Open access to data and code}
    \item[] Question: Does the paper provide open access to the data and code, with sufficient instructions to faithfully reproduce the main experimental results, as described in supplemental material?
    \item[] Answer: \answerYes{}
    \item[] Justification: We provide the code needed to reproduce the experimental results as part of the supplementary material, together with instructions for running the experiments. We will release a public GitHub repository upon acceptance.
    \item[] Guidelines:
    \begin{itemize}
        \item The answer \answerNA{} means that paper does not include experiments requiring code.
        \item Please see the NeurIPS code and data submission guidelines (\url{https://neurips.cc/public/guides/CodeSubmissionPolicy}) for more details.
        \item While we encourage the release of code and data, we understand that this might not be possible, so \answerNo{} is an acceptable answer. Papers cannot be rejected simply for not including code, unless this is central to the contribution (e.g., for a new open-source benchmark).
        \item The instructions should contain the exact command and environment needed to run to reproduce the results. See the NeurIPS code and data submission guidelines (\url{https://neurips.cc/public/guides/CodeSubmissionPolicy}) for more details.
        \item The authors should provide instructions on data access and preparation, including how to access the raw data, preprocessed data, intermediate data, and generated data, etc.
        \item The authors should provide scripts to reproduce all experimental results for the new proposed method and baselines. If only a subset of experiments are reproducible, they should state which ones are omitted from the script and why.
        \item At submission time, to preserve anonymity, the authors should release anonymized versions (if applicable).
        \item Providing as much information as possible in supplemental material (appended to the paper) is recommended, but including URLs to data and code is permitted.
    \end{itemize}

\item {\bf Experimental setting/details}
    \item[] Question: Does the paper specify all the training and test details (e.g., data splits, hyperparameters, how they were chosen, type of optimizer) necessary to understand the results?
    \item[] Answer: \answerYes{}
    \item[] Justification: The paper specifies the evaluated models, datasets, baselines, confidence thresholds, calibration procedure, and evaluation metrics.
    \item[] Guidelines:
    \begin{itemize}
        \item The answer \answerNA{} means that the paper does not include experiments.
        \item The experimental setting should be presented in the core of the paper to a level of detail that is necessary to appreciate the results and make sense of them.
        \item The full details can be provided either with the code, in appendix, or as supplemental material.
    \end{itemize}

\item {\bf Experiment statistical significance}
    \item[] Question: Does the paper report error bars suitably and correctly defined or other appropriate information about the statistical significance of the experiments?
    \item[] Answer: \answerYes{}
    \item[] Justification: The threshold-selection experiments are repeated 100 times per model-dataset combination, and the resulting plots report mean performance with one-standard-deviation bands. For risk-coverage comparisons, the paper reports aggregate AURC results across models and datasets, with 95\% bootstrapped confidence intervals.
    \item[] Guidelines:
    \begin{itemize}
        \item The answer \answerNA{} means that the paper does not include experiments.
        \item The authors should answer \answerYes{} if the results are accompanied by error bars, confidence intervals, or statistical significance tests, at least for the experiments that support the main claims of the paper.
        \item The factors of variability that the error bars are capturing should be clearly stated (for example, train/test split, initialization, random drawing of some parameter, or overall run with given experimental conditions).
        \item The method for calculating the error bars should be explained (closed form formula, call to a library function, bootstrap, etc.)
        \item The assumptions made should be given (e.g., Normally distributed errors).
        \item It should be clear whether the error bar is the standard deviation or the standard error of the mean.
        \item It is OK to report 1-sigma error bars, but one should state it. The authors should preferably report a 2-sigma error bar than state that they have a 96\% CI, if the hypothesis of Normality of errors is not verified.
        \item For asymmetric distributions, the authors should be careful not to show in tables or figures symmetric error bars that would yield results that are out of range (e.g., negative error rates).
        \item If error bars are reported in tables or plots, the authors should explain in the text how they were calculated and reference the corresponding figures or tables in the text.
    \end{itemize}

\item {\bf Experiments compute resources}
    \item[] Question: For each experiment, does the paper provide sufficient information on the computer resources (type of compute workers, memory, time of execution) needed to reproduce the experiments?
    \item[] Answer: \answerYes{}
    \item[] Justification: Appendix~\ref{app:technical_details} reports the implementation setup and hardware required for the experiments.
    \item[] Guidelines:
    \begin{itemize}
        \item The answer \answerNA{} means that the paper does not include experiments.
        \item The paper should indicate the type of compute workers CPU or GPU, internal cluster, or cloud provider, including relevant memory and storage.
        \item The paper should provide the amount of compute required for each of the individual experimental runs as well as estimate the total compute. 
        \item The paper should disclose whether the full research project required more compute than the experiments reported in the paper (e.g., preliminary or failed experiments that didn't make it into the paper). 
    \end{itemize}
    
\item {\bf Code of ethics}
    \item[] Question: Does the research conducted in the paper conform, in every respect, with the NeurIPS Code of Ethics \url{https://neurips.cc/public/EthicsGuidelines}?
    \item[] Answer: \answerYes{}
    \item[] Justification: The work uses public benchmarks, public knowledge sources, and publicly described or accessible models, and does not involve sensitive data collection or deployment. We are not aware of any aspect of the research that conflicts with the NeurIPS Code of Ethics.
    \item[] Guidelines:
    \begin{itemize}
        \item The answer \answerNA{} means that the authors have not reviewed the NeurIPS Code of Ethics.
        \item If the authors answer \answerNo, they should explain the special circumstances that require a deviation from the Code of Ethics.
        \item The authors should make sure to preserve anonymity (e.g., if there is a special consideration due to laws or regulations in their jurisdiction).
    \end{itemize}

\item {\bf Broader impacts}
    \item[] Question: Does the paper discuss both potential positive societal impacts and negative societal impacts of the work performed?
    \item[] Answer: \answerYes{}
    \item[] Justification: The paper discusses broader societal impacts, including reliability benefits and risks from overreliance on abstracted outputs.
    \item[] Guidelines:
    \begin{itemize}
        \item The answer \answerNA{} means that there is no societal impact of the work performed.
        \item If the authors answer \answerNA{} or \answerNo, they should explain why their work has no societal impact or why the paper does not address societal impact.
        \item Examples of negative societal impacts include potential malicious or unintended uses (e.g., disinformation, generating fake profiles, surveillance), fairness considerations (e.g., deployment of technologies that could make decisions that unfairly impact specific groups), privacy considerations, and security considerations.
        \item The conference expects that many papers will be foundational research and not tied to particular applications, let alone deployments. However, if there is a direct path to any negative applications, the authors should point it out. For example, it is legitimate to point out that an improvement in the quality of generative models could be used to generate Deepfakes for disinformation. On the other hand, it is not needed to point out that a generic algorithm for optimizing neural networks could enable people to train models that generate Deepfakes faster.
        \item The authors should consider possible harms that could arise when the technology is being used as intended and functioning correctly, harms that could arise when the technology is being used as intended but gives incorrect results, and harms following from (intentional or unintentional) misuse of the technology.
        \item If there are negative societal impacts, the authors could also discuss possible mitigation strategies (e.g., gated release of models, providing defenses in addition to attacks, mechanisms for monitoring misuse, mechanisms to monitor how a system learns from feedback over time, improving the efficiency and accessibility of ML).
    \end{itemize}
    
\item {\bf Safeguards}
    \item[] Question: Does the paper describe safeguards that have been put in place for responsible release of data or models that have a high risk for misuse (e.g., pre-trained language models, image generators, or scraped datasets)?
    \item[] Answer: \answerNA{}
    \item[] Justification: The paper does not release a new high-risk model, scraped dataset, or generative system requiring controlled access. It proposes and evaluates a method using existing models and public benchmarks.
    \item[] Guidelines:
    \begin{itemize}
        \item The answer \answerNA{} means that the paper poses no such risks.
        \item Released models that have a high risk for misuse or dual-use should be released with necessary safeguards to allow for controlled use of the model, for example by requiring that users adhere to usage guidelines or restrictions to access the model or implementing safety filters. 
        \item Datasets that have been scraped from the Internet could pose safety risks. The authors should describe how they avoided releasing unsafe images.
        \item We recognize that providing effective safeguards is challenging, and many papers do not require this, but we encourage authors to take this into account and make a best faith effort.
    \end{itemize}

\item {\bf Licenses for existing assets}
    \item[] Question: Are the creators or original owners of assets (e.g., code, data, models), used in the paper, properly credited and are the license and terms of use explicitly mentioned and properly respected?
    \item[] Answer: \answerYes{}
    \item[] Justification: Appendix~\ref{app:technical_details} lists the main existing assets used in the paper, including datasets, models, and knowledge sources, together with their licenses or terms of use where applicable. The corresponding papers or model cards are cited in the main text.
    \item[] Guidelines:
    \begin{itemize}
        \item The answer \answerNA{} means that the paper does not use existing assets.
        \item The authors should cite the original paper that produced the code package or dataset.
        \item The authors should state which version of the asset is used and, if possible, include a URL.
        \item The name of the license (e.g., CC-BY 4.0) should be included for each asset.
        \item For scraped data from a particular source (e.g., website), the copyright and terms of service of that source should be provided.
        \item If assets are released, the license, copyright information, and terms of use in the package should be provided. For popular datasets, \url{paperswithcode.com/datasets} has curated licenses for some datasets. Their licensing guide can help determine the license of a dataset.
        \item For existing datasets that are re-packaged, both the original license and the license of the derived asset (if it has changed) should be provided.
        \item If this information is not available online, the authors are encouraged to reach out to the asset's creators.
    \end{itemize}

\item {\bf New assets}
    \item[] Question: Are new assets introduced in the paper well documented and is the documentation provided alongside the assets?
    \item[] Answer: \answerNA{}
    \item[] Justification: The paper does not introduce a new dataset, model checkpoint, or software package as a released asset. The main contribution is a methodological framework and evaluation pipeline.
    \item[] Guidelines:
    \begin{itemize}
        \item The answer \answerNA{} means that the paper does not release new assets.
        \item Researchers should communicate the details of the dataset\slash code\slash model as part of their submissions via structured templates. This includes details about training, license, limitations, etc. 
        \item The paper should discuss whether and how consent was obtained from people whose asset is used.
        \item At submission time, remember to anonymize your assets (if applicable). You can either create an anonymized URL or include an anonymized zip file.
    \end{itemize}

\item {\bf Crowdsourcing and research with human subjects}
    \item[] Question: For crowdsourcing experiments and research with human subjects, does the paper include the full text of instructions given to participants and screenshots, if applicable, as well as details about compensation (if any)? 
    \item[] Answer: \answerNA{}
    \item[] Justification: The paper does not involve crowdsourcing experiments or research with human subjects. The reported manual validations were performed by the authors and this did not require a crowdsourced human-subjects study.
    \item[] Guidelines:
    \begin{itemize}
        \item The answer \answerNA{} means that the paper does not involve crowdsourcing nor research with human subjects.
        \item Including this information in the supplemental material is fine, but if the main contribution of the paper involves human subjects, then as much detail as possible should be included in the main paper. 
        \item According to the NeurIPS Code of Ethics, workers involved in data collection, curation, or other labor should be paid at least the minimum wage in the country of the data collector. 
    \end{itemize}

\item {\bf Institutional review board (IRB) approvals or equivalent for research with human subjects}
    \item[] Question: Does the paper describe potential risks incurred by study participants, whether such risks were disclosed to the subjects, and whether Institutional Review Board (IRB) approvals (or an equivalent approval/review based on the requirements of your country or institution) were obtained?
    \item[] Answer: \answerNA{}
    \item[] Justification: The paper does not involve crowdsourcing experiments or research with human subjects. 
    \item[] Guidelines:
    \begin{itemize}
        \item The answer \answerNA{} means that the paper does not involve crowdsourcing nor research with human subjects.
        \item Depending on the country in which research is conducted, IRB approval (or equivalent) may be required for any human subjects research. If you obtained IRB approval, you should clearly state this in the paper. 
        \item We recognize that the procedures for this may vary significantly between institutions and locations, and we expect authors to adhere to the NeurIPS Code of Ethics and the guidelines for their institution. 
        \item For initial submissions, do not include any information that would break anonymity (if applicable), such as the institution conducting the review.
    \end{itemize}

\item {\bf Declaration of LLM usage}
    \item[] Question: Does the paper describe the usage of LLMs if it is an important, original, or non-standard component of the core methods in this research? Note that if the LLM is used only for writing, editing, or formatting purposes and does \emph{not} impact the core methodology, scientific rigor, or originality of the research, declaration is not required.
    \item[] Answer: \answerYes{}
    \item[] Justification: LLM usage is central to the proposed method and evaluation pipeline.
    \item[] Guidelines:
    \begin{itemize}
        \item The answer \answerNA{} means that the core method development in this research does not involve LLMs as any important, original, or non-standard components.
        \item Please refer to our LLM policy in the NeurIPS handbook for what should or should not be described.
    \end{itemize}

\end{enumerate}

\end{document}